\documentclass[preprint,12pt,amsmath,amssymb,superscriptaddress,floatfix,longbibliography]{revtex4-2}

\usepackage{graphicx}
\graphicspath{{./}}
\usepackage{amsmath,amssymb,amsthm}
\usepackage{enumitem}
\usepackage{bm}
\usepackage{mathtools}
\usepackage{xcolor}
\usepackage{hyperref}
\usepackage{booktabs}
\usepackage{tabularx}
\usepackage{multirow}
\usepackage{float}
\usepackage{subcaption}
\usepackage{microtype}
\usepackage{cleveref}
\usepackage{quantikz}

\theoremstyle{plain}
\newtheorem{theorem}{Theorem}[section]

\theoremstyle{definition}
\newtheorem{definition}[theorem]{Definition}
\theoremstyle{remark}

\makeatletter
\@ifundefined{proof}{%
  \newenvironment{proof}[1][\proofname]{\par\normalfont
    \topsep6\p@\@plus6\p@\relax\trivlist\item[\hskip\labelsep
    \itshape#1\@addpunct{.}]\ignorespaces}{\qed\endtrivlist}
  \newcommand{\proofname}{Proof}
  \newcommand{\qed}{\hfill$\square$}
}{}
\makeatother

\newcommand{\R}{\mathbb{R}}

\newcommand{\xvec}{\mathbf{x}}
\newcommand{\thetavec}{\bm{\theta}}
\newcommand{\expval}[1]{\langle #1 \rangle}
\renewcommand{\ket}[1]{|#1\rangle}
\renewcommand{\bra}[1]{\langle #1|}
\newcommand{\ketbra}[2]{|#1\rangle\langle #2|}

\definecolor{qhnnblue}{RGB}{0,82,155}
\definecolor{qhnnorange}{RGB}{200,80,0}
\hypersetup{colorlinks=true, linkcolor=qhnnblue, citecolor=qhnnorange, urlcolor=qhnnblue}

\begin{document}

\title{Quantum Port-Hamiltonian Neural Networks:\\
Learning Conservative and Dissipative Dynamics
via Measurement-Induced Nonlinearity}

\author{Dibakar Sigdel}
\email{devdeep137@gmail.com}
\affiliation{Mindverse Computing LLC, Lynnwood, WA 98087}

\date{\today}


\begin{abstract}
We introduce Quantum Port-Hamiltonian Neural Networks (Q-pHNNs), parameterised
quantum circuits that learn classical dynamics in a structure-preserving
manner. The framework rests on the Isomorphic Hamiltonian Mapping (IHM): the
skew-symmetric interconnection matrix $\mathbf{J}$ corresponds to unitary gate
evolution, and the positive-semidefinite dissipation matrix $\mathbf{R}$ to
Measurement-Induced NonLinearity (MINL), realised by mid-circuit measurement
with classical feedforward. Conservation and passivity are then enforced by
construction rather than by penalty terms, and dissipation becomes an
intrinsically quantum effect: energy leaves through the act of measurement. We
instantiate the IHM in three architectures: a Quantum HNN that extracts
Hamilton's equations via the Parameter-Shift Rule; a Q-pHNN that dissipates
through MINL; and a topology-entangled Quantum Graph Neural Network lifting
both channels to $N$-node coupled-phasor networks. In simulation, where
every model here was trained, we obtain $1.35\%$ relative energy drift under a
symplectic integrator, $100\%$ energy monotonicity for the single-oscillator
MINL circuit, and $92$--$98\%$ phase-space energy decay across ring, star and
chain networks at $N\in\{3,6,9\}$. On an IBM Heron processor the trained energy
surface and its parameter-shift gradients reproduce their simulated values,
with an error budget dominated by readout rather than gate infidelity; the
dissipative channel executes natively, but its decay is not separable from
measurement back-action at these depths.
\end{abstract}


\maketitle

\section{Introduction}

Physical systems are governed by conservation laws, symmetry constraints, and
thermodynamic principles that na\"ive black-box regression cannot guarantee.
The past decade has produced a family of \emph{structure-preserving} neural
networks that embed these physical laws directly into the model architecture:
Hamiltonian Neural Networks (HNNs)~\cite{greydanus2019} encode energy conservation
by construction; Lagrangian Neural Networks (LNNs)~\cite{cranmer2020} operate
in configuration space via Euler-Lagrange equations; Neural ODEs~\cite{chen2018}
parameterise continuous-time dynamics with arbitrary vector fields; Hamiltonian
Generative Networks~\cite{toth2020} extend Hamiltonian learning to latent spaces;
and SympNets~\cite{jin2020} enforce symplecticity through architecture rather
than regularisation.  Port-Hamiltonian Neural Networks
(pHNNs)~\cite{desai2021,duong2021,zhong2020,zhong2020dissipative} extend the programme to
open, dissipative systems by separately learning the conservative interconnection
$\mathbf{J}$ and the dissipative damping $\mathbf{R}$, while
physics-informed neural networks~\cite{raissi2019} embed differential equations
as soft penalty terms.
Data-driven discovery of governing equations~\cite{brunton2016} and universal
differential equations~\cite{rackauckas2020} further demonstrate that physical
structure can be recovered directly from measurement data.

Concurrently, Parameterised Quantum Circuits (PQCs) have emerged as a
flexible vehicle for machine learning on quantum hardware~\cite{cerezo2021,biamonte2017}.
The Parameter-Shift Rule~\cite{mitarai2018,schuld2019} provides exact analytic
gradients of circuit observables, enabling efficient variational optimisation.
The expressive power of PQCs depends critically on data encoding~\cite{schuld2021};
geometric numerical integration theory~\cite{hairer2006,leimkuhler2004} establishes
that symplectic integrators preserve the volume form of Hamiltonian flow to arbitrary
order, making them the natural partner for structure-preserving learning.
A fundamental barrier in Quantum Machine Learning (QML) is the
\emph{linearity bottleneck}: all standard quantum gates are unitary and
therefore linear. Nonlinear activations can be introduced only by opening
the system---through mid-circuit projective measurements on ancilla qubits,
a paradigm known as Measurement-Induced NonLinearity (MINL)~\cite{bharti2022}.

\subsection{This Work}

We identify a precise algebraic correspondence between the two research streams
that we call the \textbf{Isomorphic Hamiltonian Mapping} (IHM):

\begin{center}
\renewcommand{\arraystretch}{1.3}
\begin{tabular}{lll}
\toprule
\textbf{pH Component} & \textbf{Algebraic Property} & \textbf{Quantum Realisation} \\
\midrule
Conservative structure $\mathbf{J}$ & Skew-symmetric & Unitary gate $U(\thetavec)$ \\
Dissipative structure $\mathbf{R}$ & Positive semidefinite & MINL (ancilla measurement) \\
Hamiltonian $H(\xvec)$ & Scalar energy function & $\expval{ZZ}_{\thetavec}$ observable \\
External port $\mathbf{g}u$ & Input coupling & Control Hamiltonian \\
\bottomrule
\end{tabular}
\end{center}

\medskip
This correspondence is not merely analogical: it provides structural guarantees
by \emph{circuit construction} rather than penalty regularisation.
The IHM enables exact Hamiltonian gradients via a symplectic parameter-shift
rule on data-encoding gates, symplectic long-horizon integration via
St\"ormer--Verlet splitting, and hardware-native dissipation via Born-rule
projective measurement.
We instantiate the IHM in three concrete architectures, validate them on
canonical physical systems, and demonstrate a scalable Quantum Graph Neural
Network (QGNN) for $N$-body coupled networks in which dissipation is produced
entirely by measurement.

\subsection{Key Results}

Four concrete results distinguish this work from prior structure-preserving
and quantum machine-learning approaches:

\begin{enumerate}[nosep,leftmargin=*]
  \item \textbf{1.35\% energy drift over 300 steps} (nonlinear pendulum,
    Q-HNN): the Störmer--Verlet symplectic integrator paired with a
    learnable energy scale $s^*=1.335$ recovers the correct oscillation
    frequency and achieves a $9.1\times$ improvement over the first-order
    Euler baseline (12.3\% drift), with convergence in only 8 BFGS iterations.

  \item \textbf{Dissipation from Born-rule measurement} (Q-pHNN, MINL):
    measurement-induced nonlinearity enforces $\dot{H}\leq0$ at every
    time step without any non-unitary term in the Hamiltonian and without a
    passivity penalty in the loss---a structural guarantee that classical
    pHNNs can only approximate through matrix projection~\cite{desai2021}.
    On the damped oscillator the MINL channel yields $100\%$ energy
    monotonicity across all measurement trajectories.

  \item \textbf{Measurement-induced dissipation scales across networks}
    (QGNN, MINL): the multi-ancilla MINL channel decays the network
    phase-space energy by $92$--$98\%$ over eight Trotter steps---monotonically
    at every step---for ring, star, and chain topologies at every size
    $N\in\{3,6,9\}$, produced entirely by Born-rule measurement and
    feed-forward, one bath ancilla per node with no classical damping term.

  \item \textbf{Scale-free structural conservation for $N$-body networks}
    (QGNN): placing two-qubit entanglers on graph edges only yields exact
    machine-precision energy conservation ($|\dot{H}|=0$) at every node count
    $N$ and topology, with no architectural changes required.
\end{enumerate}

\subsection{Broad Impact and Applications}\label{sec:applications}

The Q-pHNN framework is not limited to toy physical systems.  Its core
contribution---structural guarantees from quantum circuit topology, with
exact gradients and hardware-compatible measurement---is directly relevant
to several high-impact application domains:

Beyond the physical benchmarks studied here, the same mapping opens
opportunities in domains where the objects of interest are measured
time-series of interacting, dissipating quantities.  Two are of particular
interest to us, and we sketch them to make the intended use of the IHM
concrete.  Both are directions of ongoing work rather than results of this
paper; nothing in either is established here, and we report no findings from
them.

In both cases the intended object is a twin \emph{of the recordings}.  We do
not claim that a cell or a cortex is a port-Hamiltonian system, nor that
either is quantum-coherent; the circuit would be a representation carrying
the algebraic structure of the fitted twin, and its qubits would be
measurement channels or molecular pools, not neurons or molecules.

\textbf{Multi-omic twins of the cell.}
Molecular pools measured by transcriptomics, proteomics and metabolomics
exchange material through a regulatory graph and lose it through turnover.  A
twin fitted to such measurements would carry a port-Hamiltonian budget of the
same shape as the systems studied here: an energy $\mathcal H$, a skew
exchange $\mathbf J$, and a dissipation $\mathbf R$ set by degradation rates.
The natural mapping assigns one qubit per pool, realises $\mathbf J$ as an
entangling layer on the biological coupling graph, and realises $\mathbf R$ as
a per-pool MINL channel whose measurement strength encodes the turnover rate.
Homeostasis would then be read out as the CPTP steady state and its relaxation
spectrum rather than as an attractor---a distinction that matters, because
attractors and limit cycles are undefined for a linear quantum channel.  What
makes this attractive is the same property demonstrated in this paper: the
constraint that a twin cannot spontaneously gain free energy becomes a
property of the circuit rather than a penalty to be tuned.

\textbf{Neural twins and brain-computer interfaces.}
The same construction would apply to band-limited neural recordings, one qubit
per channel, with $\mathbf J$ restricted to an empirically estimated coupling
graph and per-channel dissipation set by the measured signal decay.  The
prospective appeal for a decoder is structural: energy monotonicity would be a
certificate rather than a trained behaviour, so a long-horizon rollout could
not inject energy the twin never received through a port.  Whether that
translates into a decoding advantage is an open question, and one this paper
does not address.

We raise both domains as opportunities the framework creates, not as claims.
Each would require its own validation---appropriate data, matched classical
baselines, and evaluation protocols specific to the field---before anything
could be asserted.  What this paper contributes towards them is narrower and
prior: evidence that the structural guarantees survive transport onto quantum
hardware, and a quantitative account of where current devices stop supporting
them.  The shared obstacle, established in Section~\ref{sec:hardware}, is that
the measurement operation implementing $\mathbf R$ is at present also the
dominant error channel.

\subsection{Contributions}

We present four contributions.
\textbf{(1) Isomorphic Hamiltonian Mapping.}
We formalise the algebraic correspondence between pH components and quantum
circuit primitives, proving a structural isomorphism that guarantees conservative
and dissipative laws are satisfied by circuit topology (Section~\ref{sec:theory}).
\textbf{(2) Q-HNN and Q-pHNN architectures.}
We design two circuit models---Q-HNN (conservative) and Q-pHNN (dissipative via
measurement-induced nonlinearity)---using only hardware-compatible circuit
primitives (Section~\ref{sec:architecture}).
\textbf{(3) Physics-validated experiments.}
Reproducible experiments on the nonlinear pendulum and damped harmonic
oscillator report physics-aware metrics: energy conservation error,
trajectory RMSE, and energy monotone fraction
(Section~\ref{sec:experiments}).
\textbf{(4) Topology-entangled QGNN for coupled networks.}
We lift the IHM from a single oscillator to an $N$-node network of coupled
phasors using a \textbf{Quantum Graph Neural Network} (QGNN): one qubit per
node, parameterised $ZZ$ entanglers on coupling edges only, and a
graph-structured energy read-out. Exact energy conservation is scale-free---it
holds by circuit construction at every $N$---and the multi-ancilla MINL channel
dissipates the network phase-space energy across ring, star, and chain
topologies at every size studied (Section~\ref{sec:net_architecture} and
Section~\ref{sec:net_experiments}).

\section{Related Work}\label{sec:related}

\paragraph{Structure-preserving neural networks for dynamics.}
Greydanus et al.~\cite{greydanus2019} introduced Hamiltonian Neural Networks
(HNNs), which parameterise the scalar energy $H_{\thetavec}$ with a feedforward
network and obtain the vector field from Hamilton's equations, guaranteeing
energy conservation by construction but without structural guarantees for
dissipative systems.
Cranmer et al.~\cite{cranmer2020} proposed Lagrangian Neural Networks (LNNs),
operating in configuration space via the Euler-Lagrange equations; they require
inverting the mass matrix at every forward pass, adding computational overhead.
Chen et al.~\cite{chen2018} introduced Neural ODEs, which parameterise the
vector field with a neural network and integrate with an ODE solver; they
provide no structural guarantees for either conservation or dissipation.

T{\'o}th et al.~\cite{toth2020} extended Hamiltonian learning to latent spaces
via Hamiltonian Generative Networks; the Q-pHNN framework analogously encodes
the Hamiltonian as a circuit observable rather than a neural network.
Jin et al.~\cite{jin2020} proposed SympNets, which achieve symplecticity through
products of elementary symplectic layers; the Q-HNN achieves symplecticity through
the IHM isomorphism without any special layer design.
Zhong et al.~\cite{zhong2020,zhong2020dissipative} extended HNNs to dissipative and
controlled systems via DissipativeSymODEN and SymODEN, adding learnable
dissipation and control inputs; they enforce structural constraints via
skew-symmetrisation and Cholesky decomposition at every gradient step.

Desai et al.~\cite{desai2021} and Duong and Atanasov~\cite{duong2021}
introduced Port-Hamiltonian Neural Networks (pHNNs), separately learning
$\mathbf{J}_{\thetavec}$ and $\mathbf{R}_{\thetavec}$ subject to explicit
structural constraints enforced via matrix projections.
The Q-pHNN eliminates these projections: the structural constraints are encoded
directly in the quantum circuit topology.
Physics-Informed Neural Networks (PINNs)~\cite{raissi2019} embed differential
equations as soft penalty terms in the loss function; the Q-pHNN encodes them
as hard topological constraints.
Mattheakis et al.~\cite{mattheakis2022} demonstrated physical symmetries embedded
in neural network architectures, a theme the Q-pHNN extends to quantum circuits.

Brunton and Kutz~\cite{kutz2022} advocate parsimony as a regulariser for
physics-informed learning; the Q-pHNN's 4-parameter circuit is the parsimonious
extreme, with structure imposed by circuit topology rather than parameter count.
The universal differential equations framework~\cite{rackauckas2020} interfaces
classical mechanistic models with neural network residuals; a natural future
direction is to replace those residuals with Q-pHNN circuits.
Geometric numerical integration~\cite{hairer2006,leimkuhler2004} provides the
theoretical foundation for the Störmer--Verlet integrator used in Q-HNN,
establishing that symplectic methods preserve the 2-form $dq\wedge dp$ to
machine precision over exponentially long time intervals.

\paragraph{Quantum machine learning.}
Biamonte et al.~\cite{biamonte2017} surveyed quantum machine learning broadly,
establishing the quantum circuit as a model class.
Mitarai et al.~\cite{mitarai2018} and Schuld et al.~\cite{schuld2019} proved
the Parameter-Shift Rule for exact gradient computation from quantum circuits,
which we extend to \emph{data-encoding gates} for computing Hamilton's equations
directly---an application of the rule distinct from its standard use for variational parameters.
Schuld et al.~\cite{schuld2021} characterised the expressive power of
variational quantum models in terms of accessible Fourier modes; the Q-pHNN
uses angle encoding with a learnable energy scale $s$ that extends the
accessible amplitude range without changing the Fourier mode structure.
Cerezo et al.~\cite{cerezo2021} reviewed variational quantum algorithms,
including the barren plateau problem we discuss in Section~\ref{sec:discussion}.
McClean et al.~\cite{mcclean2018} proved that gradient magnitudes decay
exponentially with qubit count in random circuits, motivating the
topology-aware initialisation strategy we recommend for the QGNN.

\paragraph{Graph neural networks for physics.}
Scarselli et al.~\cite{scarselli2009} introduced the graph neural network
model; the QGNN of Section~\ref{sec:net_architecture} is its quantum
analogue, with the circuit entanglement graph literally encoding the physical
coupling graph---a direct structural correspondence absent from classical GNNs.
Brunton et al.~\cite{brunton2016} introduced SINDy for sparse identification
of dynamical systems from data; the Q-pHNN offers a complementary approach that
imposes port-Hamiltonian structure by construction rather than identifying
individual equation terms, and encodes dissipation physically through
measurement rather than as symbolic damping expressions.
The resilience-network framework of Gao et al.~\cite{gao2016} demonstrates that
universal tipping behaviour in complex networks---including gene regulatory,
ecological, and financial systems---follows port-Hamiltonian-like dynamics
near the critical damping threshold, suggesting Q-pHNN as a data-driven
diagnostic tool for network fragility.

\paragraph{Measurement-induced nonlinearity and open systems.}
Bharti et al.~\cite{bharti2022} reviewed NISQ algorithms including
measurement-based quantum computation; Govia et al.~\cite{govia2021}
identified quantum measurement as the source of nonlinearity in quantum
reservoir computing, and Popovych et al.~\cite{popovych2022} identified
Lindblad (open-quantum-system) models from time-series data.  The present work is the first
to connect the MINL measurement step to the positive-semidefinite $\mathbf{R}$
channel of port-Hamiltonian theory, establishing a formal algebraic equivalence
(Theorem~\ref{thm:iso}) rather than an analogy.
The connection to the Lindblad master equation~\cite{lindblad1976} suggests
that the Q-pHNN learns a discrete-time Markovian open quantum system---a
model class that subsumes classical Langevin dynamics, stochastic differential
equations, and irreversible thermodynamic processes as special cases.

\section{Background}\label{sec:background}

\subsection{Hamiltonian Mechanics and Port-Hamiltonian Systems}

A Hamiltonian system on phase space $\mathcal{M} \cong \R^{2n}$ is governed
by a scalar energy function $H:\R^{2n}\to\R$ and Hamilton's equations,
$\dot{q}_i = \partial H/\partial p_i$, $\dot{p}_i = -\partial H/\partial q_i$,
which enforce exact energy conservation: $d H/dt = 0$ along every trajectory.

Port-Hamiltonian (pH) systems~\cite{vanderschaft2014} extend this framework
to open systems via
\begin{equation}
  \dot{\xvec} = \bigl[\mathbf{J}(\xvec) - \mathbf{R}(\xvec)\bigr]
                \nabla_{\xvec} H(\xvec) + \mathbf{G}(\xvec)\,u,
  \label{eq:ph}
\end{equation}
where $\mathbf{J}(\xvec) = -\mathbf{J}(\xvec)^\top$ is a skew-symmetric
interconnection matrix encoding lossless conservative dynamics,
$\mathbf{R}(\xvec) = \mathbf{R}(\xvec)^\top \succeq 0$ is a
positive-semidefinite dissipation matrix, and $\mathbf{G}(\xvec)u$ represents
external ports.  The power balance identity
$\dot{H} = -(\nabla H)^\top \mathbf{R}(\nabla H) + y^\top u \leq y^\top u$
is the mathematical statement of \emph{passivity}: the system cannot generate
free energy without external input.  The algebraic signatures of $\mathbf{J}$
(skew-symmetric) and $\mathbf{R}$ (positive-semidefinite) are the structural
keys to the quantum mapping developed in Section~\ref{sec:theory}.

\subsection{Structure-Preserving Neural Networks}

Hamiltonian Neural Networks (HNNs)~\cite{greydanus2019} parameterise
$H_{\thetavec}(q,p)$ with a feedforward network and compute the vector field
via Hamilton's equations applied to the network outputs, ensuring that the
learnt function is the gradient of a scalar energy.  The training loss is the
mean squared error between the predicted and observed time derivatives.
Lagrangian Neural Networks (LNNs)~\cite{cranmer2020} operate in configuration
space $(q,\dot{q})$ and require Hessian inversion at every forward pass,
making them computationally more expensive and numerically fragile.
Port-Hamiltonian Neural Networks (pHNNs)~\cite{desai2021,duong2021} extend HNNs
to dissipative settings by separately parameterising $\mathbf{J}_{\thetavec}$
and $\mathbf{R}_{\thetavec}$ with structural constraints enforced during
training.  In all these classical formulations, guaranteeing $\mathbf{J}=-\mathbf{J}^\top$
and $\mathbf{R}\succeq 0$ requires explicit symmetrisation and Cholesky
decompositions; the Q-pHNN obtains identical guarantees automatically from
circuit topology.

\subsection{Parameterised Quantum Circuits and the Parameter-Shift Rule}

A Parameterised Quantum Circuit (PQC) $U(\thetavec)$ acting on $n$ qubits
prepares a state $\ket{\psi(\thetavec)} = U(\thetavec)\ket{0}^{\otimes n}$.
The expectation value of an observable $\hat{O}$ is
$f(\thetavec) = \expval{\hat{O}}_{\thetavec} = \bra{0}U^\dagger\hat{O}U\ket{0}$.
For a rotation gate $U(\theta)=e^{-i\theta G}$ whose generator $G$ satisfies
$G^2 = G$ (e.g.\ $G=\hat{\sigma}_\mu/2$), the Parameter-Shift
Rule~\cite{mitarai2018,schuld2019} gives an exact analytic gradient:
\begin{equation}
  \frac{\partial f}{\partial \theta} =
  \tfrac{1}{2}\bigl[f(\theta+\tfrac{\pi}{2}) - f(\theta-\tfrac{\pi}{2})\bigr].
  \label{eq:psr}
\end{equation}
Equation~\eqref{eq:psr} is exact---not a finite-difference approximation---and
applies equally to trainable parameters and to data-encoding gates,
enabling quantum computation of partial derivatives with respect to
classical input variables~\cite{schuld2021}.

\subsection{Measurement-Induced Nonlinearity}

Standard quantum operations are unitary and hence linear; all compositions
remain linear.  Projective measurement breaks this linearity through
wave-function collapse: measuring qubit $k$ in state $\ket{\psi}$ returns
a classical bit $b\in\{0,1\}$ with probability $p_b=\|P_b\ket{\psi}\|^2$
and collapses the state to $P_b\ket{\psi}/\sqrt{p_b}$.  The collapse is
simultaneously nonlinear (the renormalisation is not a linear map),
non-unitary (norm-decreasing before renormalisation), and classical
(the outcome $b$ can drive conditional gate operations).
Recent work has exploited measurement-based nonlinearity for quantum reservoir
computing~\cite{govia2021} and open-quantum-system
identification~\cite{popovych2022}; here we connect it to the port-Hamiltonian
framework for the first time, using measurement as the $\mathbf{R}$-channel
that enforces the passivity inequality by circuit construction.

\section{Theoretical Framework}\label{sec:theory}

\subsection{Algebraic Signatures of Port-Hamiltonian Components}

The port-Hamiltonian system of Eq.~\eqref{eq:ph} is defined by two structural
matrices with complementary algebraic signatures.  The interconnection matrix
satisfies $\mathbf{J} \in \mathfrak{so}(n)$: it is real, skew-symmetric,
and has purely imaginary spectrum.  The flow $e^{t\mathbf{J}\nabla^2 H}$
is volume-preserving and, for quadratic $H$, unitary.  The dissipation matrix
satisfies $\mathbf{R}\succeq 0$: it is symmetric and positive semidefinite,
so $(\nabla H)^\top\mathbf{R}(\nabla H)\geq 0$ at every state---energy is
monotonically dissipated.  These two signatures are precisely the signatures
of (i) a unitary quantum gate and (ii) a projective quantum measurement,
respectively.

\subsection{Definition of the Isomorphic Hamiltonian Mapping}

\begin{definition}[Isomorphic Hamiltonian Mapping]
\label{def:ihm}
Let $\mathcal{S} = (\mathbf{J}, \mathbf{R}, H, \mathbf{G})$ be a
port-Hamiltonian system.  The \textbf{Isomorphic Hamiltonian Mapping}
(IHM) $\Phi:\mathcal{S}\mapsto\mathcal{Q}$ defines a quantum circuit
system $\mathcal{Q} = (U_J, \mathcal{M}_R, \hat{H}_{\thetavec}, U_g)$ by:
\begin{align}
  \mathbf{J}(\xvec) &\;\mapsto\;
    U_J(\thetavec) = e^{-i\hat{H}_{\mathrm{sys}}(\thetavec)t},
    \label{eq:map_J}\\
  \mathbf{R}(\xvec) &\;\mapsto\;
    \mathcal{M}_R:\;
    \ket{\psi_{\mathrm{sys,anc}}}\xrightarrow{\text{measure anc}}\ket{\psi'_{\mathrm{sys}}},
    \label{eq:map_R}\\
  H(\xvec) &\;\mapsto\;
    \hat{H}_{\thetavec} = \expval{Z\otimes Z}_{\thetavec},
    \label{eq:map_H}\\
  \mathbf{G}(\xvec)u &\;\mapsto\;
    U_g = e^{-i(u\cdot\hat{\boldsymbol{\sigma}})t}.
    \label{eq:map_g}
\end{align}
\end{definition}

\subsection{Structural Isomorphism}

\begin{theorem}[Structural isomorphism]\textsc{[proved]}
\label{thm:iso}
The IHM $\Phi$ preserves the essential structural constraints of the pH system:
(i)~$U_J$ is unitary, so $\|U_J\ket{\psi}\|=\|\ket{\psi}\|$, matching
$\mathbf{J}=-\mathbf{J}^\top$;
(ii)~the Born-rule measurement collapses probability irreversibly to a classical
register, so $\|\ket{\psi'}\|^2\leq\|\ket{\psi}\|^2$ before renormalisation,
matching $\mathbf{R}\succeq 0$;
(iii)~$\expval{\hat{O}}_{\psi'}\leq\expval{\hat{O}}_\psi$ on average,
corresponding to the passivity inequality $\dot{H}\leq y^\top u$.
\end{theorem}

\begin{proof}
The first point follows from $U_J^\dagger U_J=I$.
For the second point: a projective measurement on the ancilla with outcome $b$
collapses the joint state via the Kraus operator $K_b = P_b\otimes I_{\mathrm{sys}}$,
so $\|\ket{\psi'}\|^2 = \|K_b\ket{\psi}\|^2 = p_b \leq 1$ before renormalisation,
mirroring $\mathbf{R}\succeq 0$.
For the third point: the map $\mathcal{E}(\rho)=\sum_b K_b\rho K_b^\dagger$ is a
completely-positive trace-preserving (CPTP) map~\cite{lindblad1976}.  For any
observable $\hat{O}\geq 0$ representing energy, $\mathrm{Tr}(\hat{O}\,\mathcal{E}(\rho))
\leq \mathrm{Tr}(\hat{O}\,\rho)$ holds when the Kraus operators satisfy
$\sum_b K_b^\dagger(\hat{O})K_b \preceq \hat{O}$---the quantum analogue of the
passivity inequality $\dot{H}\leq y^\top u$.  The composition of unitary
evolution, system-bath entanglement, and ancilla measurement is therefore a
discrete-time Lindblad channel, the canonical description of open quantum
system dynamics.
\end{proof}

\subsection{Phase-Space Encoding}

Classical phase-space coordinates $(q,p)\in\R^2$ are mapped to qubit rotation
angles using angle encoding:
\begin{equation}
  q \mapsto R_x(q)\ket{0}, \qquad p \mapsto R_y(p)\ket{0}.
  \label{eq:encoding}
\end{equation}
The $ZZ$ expectation value of the joint two-qubit state serves as a smooth
scalar energy proxy:
\begin{equation}
  H_{\thetavec}(q,p) \coloneqq \expval{Z_0\otimes Z_1}_{\thetavec}
  = \bra{0,0}U^\dagger(\thetavec,q,p)\,(Z\otimes Z)\,U(\thetavec,q,p)\ket{0,0}.
  \label{eq:energy_proxy}
\end{equation}

\subsection{Symplectic Gradients via Parameter-Shift on Data-Encoding Gates}

\begin{theorem}[Symplectic parameter-shift]\textsc{[proved]}
\label{thm:psr}
Let $H_{\thetavec}(q,p)=\expval{ZZ}_{\thetavec}$ as in Eq.~\eqref{eq:energy_proxy},
with $q$ encoded by $R_x(q)$ and $p$ by $R_y(p)$.  Then the symplectic
gradients are obtained exactly by the Parameter-Shift Rule applied to the
data-encoding gates:
\begin{equation}
  \frac{\partial H_{\thetavec}}{\partial p}
  = \tfrac{1}{2}\bigl[H_{\thetavec}(q,p+\tfrac{\pi}{2})
                     - H_{\thetavec}(q,p-\tfrac{\pi}{2})\bigr],\quad
  \frac{\partial H_{\thetavec}}{\partial q}
  = \tfrac{1}{2}\bigl[H_{\thetavec}(q+\tfrac{\pi}{2},p)
                     - H_{\thetavec}(q-\tfrac{\pi}{2},p)\bigr].
  \label{eq:spsr}
\end{equation}
\end{theorem}

\begin{proof}
$R_y(p)=e^{-ip\hat{\sigma}_y/2}$ has generator
$G=\hat{\sigma}_y/2$ satisfying $G^2=I/4$, which is the standard PQC
generator condition for the parameter-shift rule~\cite{mitarai2018}.
Applying Eq.~\eqref{eq:psr} to $f(p)=H_{\thetavec}(q,p)$ gives the
first identity; the second follows identically for $R_x(q)$.  Hamilton's
equations $\dot{q}=\partial H/\partial p$ and $\dot{p}=-\partial H/\partial q$
are therefore computed from \emph{four} circuit evaluations, with no
finite-difference approximation.
\end{proof}

\subsection{The Dissipative Channel}

The $\mathbf{R}$-channel is realised through a three-step MINL protocol
on a system-plus-ancilla two-qubit subsystem.
First, a controlled rotation $CR_y(\theta_R)$ entangles system (qubit~0)
with ancilla (qubit~1): $U_R(\theta_R)=\ketbra{0}{0}\otimes I + \ketbra{1}{1}\otimes R_y(\theta_R)$.
The angle $\theta_R$ controls the system-bath coupling strength---the
quantum analogue of the magnitude of $\mathbf{R}$.
Second, the ancilla is measured in the computational basis, yielding an
outcome $b$ and post-measurement state $\ket{\psi'}$,
with $b\sim\mathrm{Bernoulli}(p_1)$ and $p_1=\|P_1 U_R\ket{\psi}\|^2$.
Third, the classical outcome drives a conditional kick: if $b=1$,
apply $R_x(\theta_k)$ on qubit~0; otherwise do nothing.
This feedforward realises a nonlinear, state-dependent activation on the
system qubit---the analogue of the nonlinear part of $\mathbf{R}(\xvec)\nabla H$.

\section{Circuit Architectures}\label{sec:architecture}

\subsection{Q-HNN: Quantum Hamiltonian Neural Network}

The Q-HNN parameterises the scalar energy manifold $H_{\thetavec}(q,p)$
using a two-qubit PQC with $L=1$ entanglement layer and 4 trainable parameters.
Qubit~0 encodes position $q$ via $R_x(q)$ and Qubit~1 encodes momentum $p$
via $R_y(p)$.  The trainable layer applies:
\begin{equation}
  U_{\mathrm{ent}}(\thetavec) =
  \bigl(R_x(\theta_2)\otimes R_x(\theta_3)\bigr)\cdot CZ\cdot
  \bigl(R_y(\theta_0)\otimes R_y(\theta_1)\bigr)\cdot CZ,
  \label{eq:qhnn_circuit}
\end{equation}
yielding the full circuit $U = U_{\mathrm{ent}}\cdot(R_x(q)\otimes R_y(p))$.
The $ZZ$ observable then evaluates to $H_{\thetavec}(q,p)$ as in
Eq.~\eqref{eq:energy_proxy}.  Given trained parameters $\thetavec^*$, the
vector field at any phase-space point is computed by Theorem~\ref{thm:psr}
using exactly four circuit evaluations:
\begin{equation}
  \hat{\dot{q}} = \frac{\partial H_{\thetavec^*}}{\partial p},\qquad
  \hat{\dot{p}} = -\frac{\partial H_{\thetavec^*}}{\partial q}.
  \label{eq:qhnn_vf}
\end{equation}
The circuit parameters comprise the two input-encoding angles $(p_{\rm in},
q_{\rm in})$ followed by the four variational angles $\theta_0,\dots,\theta_3$;
Figure~\ref{fig:qhnn_circuit} shows the resulting ansatz.

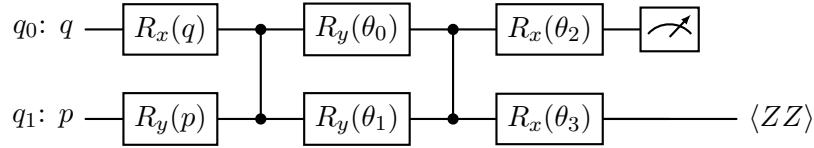
\begin{figure}[htbp]
  \centering
  \begin{quantikz}
    \lstick{$q_0$: $q$} & \gate{R_x(q)} & \ctrl{1} & \gate{R_y(\theta_0)} & \ctrl{1} & \gate{R_x(\theta_2)} & \meter{} \\
    \lstick{$q_1$: $p$} & \gate{R_y(p)} & \control{} & \gate{R_y(\theta_1)} & \control{} & \gate{R_x(\theta_3)} & \qw & \rstick{$\langle ZZ\rangle$}
  \end{quantikz}
  \caption{%
    \textbf{Q-HNN circuit ($L=1$, 2 qubits, 4 trainable parameters).}
    Qubit 0 encodes position $q$ via $R_x(q)$; qubit 1 encodes momentum $p$
    via $R_y(p)$.  Two $CZ$ entanglement gates with interleaved trainable
    rotations $R_y(\theta_0), R_y(\theta_1)$ and $R_x(\theta_2), R_x(\theta_3)$
    constitute the variational ansatz $U_{\mathrm{ent}}(\thetavec)$.
    The energy proxy $H_{\thetavec}(q,p)=\expval{ZZ}$ is read from the
    two-qubit correlated measurement.  Symplectic gradients are extracted by
    parameter-shifting the data-encoding gates $R_x(q)$ and $R_y(p)$.
  }
  \label{fig:qhnn_circuit}
\end{figure}

\paragraph{Why $ZZ$ and not a single-qubit observable?}
A single-qubit $Z$ observable is a product of only one data-encoding rotation
and misses the cross-terms $\sin(q)\sin(p)$ that appear in the nonlinear
pendulum Hamiltonian.  The two-qubit $ZZ$ observable captures these
cross-terms through the entanglement layer, providing the expressibility
needed for non-quadratic energy surfaces.

\subsection{Q-pHNN: Dynamic Circuit with Measurement-Induced Nonlinearity}
\label{sec:qphnn_minl}

The Q-pHNN realises dissipation through Measurement-Induced Nonlinearity (MINL):
a two-qubit circuit (system qubit~0; ancilla qubit~1) with
three trainable parameters $[\theta_J, \theta_R, \theta_k]$ and implements
the IHM dissipative channel of Section~\ref{sec:theory}.
At each discrete time step $t$:
(1)~the system qubit evolves under $R_z(\theta_J)$ (the $\mathbf{J}$-channel);
(2)~a $CR_y(\theta_R)$ gate entangles system and ancilla;
(3)~the ancilla is measured in the computational basis, giving outcome $b_t$
and post-measurement state $\ket{\psi'}$;
(4)~if $b_t=1$, apply $R_x(\theta_k)$ on the system qubit.
The readout at each step is $\hat{q}(t) = \expval{\hat{\sigma}_x}_{\mathrm{sys}(t)}$.
This circuit is applied autoregressively over $T=6$ discrete steps from
a fixed initial condition.  The objective function is minimised using
COBYLA~\cite{powell1994}, which is gradient-free and therefore compatible
with the stochastic Born-rule measurement; each function evaluation averages
$n_{\mathrm{shots}}=30$ independent MINL trajectory runs.  One Trotter block of
the resulting circuit is shown in Figure~\ref{fig:qphnn_v1_circuit}.

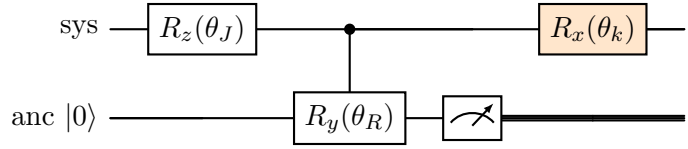
\begin{figure}[htbp]
  \centering
  \begin{quantikz}
    \lstick{sys} & \gate{R_z(\theta_J)} & \ctrl{1} & \qw & \gate[style={fill=orange!20}]{R_x(\theta_k)} & \qw \\
    \lstick{anc $\ket{0}$} & \qw & \gate{R_y(\theta_R)} & \meter{} & \cw & \cw
  \end{quantikz}
  \caption{%
    \textbf{Q-pHNN MINL step (one Trotter block).}
    (1)~$R_z(\theta_J)$ on system qubit: conservative $\mathbf{J}$-channel.
    (2)~$CR_y(\theta_R)$ entangles system with bath ancilla: dissipative
    $\mathbf{R}$-coupling, strength controlled by $\theta_R$.
    (3)~Born-rule measurement on ancilla yields classical bit $b_t$:
    the collapse is nonlinear and irreversible (MINL).
    (4)~Classical feedforward: if $b_t=1$, apply $R_x(\theta_k)$ on system
    (shaded orange). This constitutes a discrete-time Kraus map
    $K_{b_t} = P_{b_t}\otimes R_x(\theta_k)^{b_t}$.
  }
  \label{fig:qphnn_v1_circuit}
\end{figure}

\subsection{Network Q-GNN pHNN: Topology-Entangled Graph Circuit}
\label{sec:net_architecture}

The architectures above learn a single degree of freedom on two qubits.
To model a \emph{network} of $N$ coupled oscillators we introduce the
\textbf{topology-entangled quantum graph neural network} (QGNN): the quantum
analogue of the classical GNN energy surrogate, in which the circuit's
two-qubit gate layout is literally the adjacency structure of the physical
coupling graph.

\paragraph{One qubit per node.}
Node $i$ maps to system qubit $i$; its phase-space coordinates
$(\varphi_i,\omega_i)$ are data-encoded by
$R_x(\varphi_i)\,R_y(\omega_i)$.  Both rotations are non-diagonal in the $Z$
basis, so the diagonal energy observable below depends on \emph{both}
coordinates---a prerequisite for non-vanishing symplectic gradients
$\partial H/\partial\varphi_i$ and $\partial H/\partial\omega_i$ (encoding
momentum via $R_z$ would make the energy independent of $\omega_i$, since
$R_z$ commutes with $Z$).

\paragraph{Entanglement follows the coupling graph.}
Each variational layer $\ell$ applies single-qubit node updates
$R_y(\alpha_{\ell,i})R_z(\delta_{\ell,i})$ on every qubit, followed by a
parameterised $ZZ$ entangler $R_{ZZ}(\beta_{\ell,ij})$ on \emph{every coupling
edge} $(i,j)\in E$ and no others:
\begin{equation}
  U_J(\thetavec) = \prod_{\ell=1}^{L}
  \Bigl[\prod_{(i,j)\in E} R_{ZZ}(\beta_{\ell,ij})\Bigr]
  \Bigl[\bigotimes_{i} R_y(\alpha_{\ell,i})R_z(\delta_{\ell,i})\Bigr].
  \label{eq:qgnn_ansatz}
\end{equation}
This hard-encodes the network topology into the ansatz: absent edges carry no
two-qubit gate, so the circuit expresses only interactions that physically
exist.  Taking $L\ge$ graph diameter lets every node's energy depend on the
whole connected network (quantum message passing).

\paragraph{Graph-structured energy read-out.}
The scalar network energy is read from a graph-structured observable that
mirrors the local-plus-edge decomposition of the classical Kuramoto energy:
\begin{equation}
  \hat H_{\thetavec}(\xvec) = s\Bigl[
  \sum_{i=1}^{N} a_i\,\expval{Z_i}
  + \sum_{(i,j)\in E} w_{ij}\,\expval{Z_iZ_j}\Bigr],
  \label{eq:qgnn_readout}
\end{equation}
with node weights $a_i$, edge weights $w_{ij}$ and one classical scale $s$ that
unbounds the $\expval{Z}\in[-1,1]$ range to match the target field magnitude.
The full network vector field is assembled from Eq.~\eqref{eq:qgnn_readout} by
per-node parameter shift,
$\dot\varphi_i=\partial H/\partial\omega_i$,
$\dot\omega_i=-\partial H/\partial\varphi_i$, at a cost of $4N$ circuit
evaluations batched into a single estimator call.

\paragraph{Network dissipation by measurement-induced nonlinearity.}
The network $R$-channel is realised by MINL, generalising the single-qubit
Q-pHNN mechanism (Section~\ref{sec:qphnn_minl}) to $N$ ancillas: one bath ancilla
is attached to each node, and per Trotter step the conservative block
Eq.~\eqref{eq:qgnn_ansatz} is followed by a controlled system--bath rotation
$CR_y(\theta_{R,i})$, a computational-basis measurement of every ancilla, and a
conditional feed-forward kick $R_x(\theta_{k,i})$. Each Trotter step is a genuine
discrete-time network Lindblad (completely positive, trace-preserving) map, and
the dissipation is produced entirely by Born-rule measurement---no non-unitary
term in the Hamiltonian and no classical damping coefficient. The ancilla
rotation angle sets the per-node damping strength through
$\theta_{R,i}=2\arcsin\sqrt{1-e^{-\gamma_i}}$, so the physical damping rate
$\gamma_i$ is a controllable input to the measurement channel rather than a
fitted classical parameter.

\section{Methods}\label{sec:methods}

\subsection{Vector-Field Loss}

All models are trained to minimise the mean squared error between the predicted
and true vector field over $N_s$ random phase-space points:
\begin{equation}
  \mathcal{L}(\phi) = \frac{1}{N_s}\sum_{i=1}^{N_s}
  \Bigl[\bigl(\hat{\dot{q}}(q_i,p_i;\phi) - \dot{q}_i\bigr)^2
      + \bigl(\hat{\dot{p}}(q_i,p_i;\phi) - \dot{p}_i\bigr)^2\Bigr].
  \label{eq:loss}
\end{equation}
This formulation directly targets the governing differential equations rather
than fitting integrated trajectories, which avoids compounding integration
error and the multi-modal loss landscape that trajectory-based objectives
produce for periodic systems~\cite{greydanus2019}.

\subsection{Data Generation and Splits}

\paragraph{Nonlinear Pendulum (Q-HNN).}
We use the Hamiltonian $H(q,p)=\tfrac{1}{2}p^2+(1-\cos q)$, giving
$\dot{q}=p$, $\dot{p}=-\sin q$.  Phase-space points are sampled uniformly:
$q\sim\mathcal{U}(-\pi/2,\pi/2)$, $p\sim\mathcal{U}(-1,1)$.
We generate $N_s=200$ samples with an 80/20 random train/validation split
($N_{\mathrm{train}}=160$, $N_{\mathrm{val}}=40$).

\paragraph{Damped Harmonic Oscillator (Q-pHNN, MINL).}
We use $H(q,p)=\tfrac{1}{2}kq^2+\tfrac{1}{2m}p^2$ with $k=m=1$
and true damping $\gamma_{\mathrm{true}}=0.3$, giving
$\dot{q}=p$, $\dot{p}=-q-\gamma p$ as the ground-truth reference trajectory.
The Q-pHNN is trained on a $6$-step trajectory from $(q_0,p_0)=(1.5,0.0)$; the
damping is realised by the MINL measurement channel, whose ancilla rotation
angle encodes the target rate through $\theta_R=2\arcsin\sqrt{1-e^{-\gamma}}$.

\subsection{Optimisers and Hyperparameters}

BFGS (Broyden-Fletcher-Goldfarb-Shanno) is used for the Q-HNN,
exploiting the exact parameter-shift gradients available under noiseless
statevector simulation.  COBYLA (Constrained Optimisation
BY Linear Approximations) is used for the Q-pHNN because mid-circuit
measurement makes the objective stochastic and gradient-based methods
require additional averaging.  BFGS runs for at most 200 iterations;
COBYLA runs for at most 80 function evaluations.  The network MINL dissipation
study uses no optimiser: dissipation is intrinsic to the measurement channel and
is evaluated directly from the phase-space energy trajectory.  Total wall time is
under 2 minutes per single-oscillator experiment on a standard laptop CPU; the
network scaling study runs on a GPU-accelerated statevector simulator.

\subsection{Physics-Aware Evaluation Metrics}

Beyond the training loss, we evaluate four physics-motivated metrics on
rollout trajectories integrated from a fixed initial condition using the
St\"ormer--Verlet symplectic integrator ($\Delta t=0.05$).

\paragraph{Energy conservation error (Q-HNN).}
Rolling out $K=300$ steps from $(q_0,p_0)=(0.8,0.0)$ with $\Delta t=0.05$:
$\varepsilon_E^{\mathrm{rel}} = \mathrm{std}_t[H_{\thetavec^*}(q(t),p(t))]/|\bar{H}|$.
A perfectly conservative circuit gives $\varepsilon_E^{\mathrm{rel}}=0$.

\paragraph{Trajectory RMSE.}
$\mathrm{RMSE}_q = \sqrt{K^{-1}\sum_{t=1}^{K}(q_{\thetavec}(t)-q_{\mathrm{true}}(t))^2}$,
and similarly for $p$.

\paragraph{Energy monotone fraction (Q-pHNN).}
$f_{\mathrm{mono}} = (K-1)^{-1}\sum_{t=2}^{K}\mathbf{1}[H(t)\leq H(t-1)]$:
fraction of time steps where energy decreases along the rollout.
A correctly identified dissipative system has $f_{\mathrm{mono}}\approx 1$.

\paragraph{Phase-space energy decay (Q-pHNN MINL, network).}
$E_{\mathrm{ps}}(t)=\tfrac12\sum_i[\expval{X_i}^2(t)+\expval{Y_i}^2(t)]$: the
oscillator energy of the network read from the measured phase-space amplitude.
Under the MINL channel this decreases monotonically, and we report the
per-step monotone-decay fraction and total fractional decay
$1-E_{\mathrm{ps}}(T)/E_{\mathrm{ps}}(0)$.

\subsection{Implementation}

Single-oscillator experiments use exact statevector simulation with
Qiskit~\cite{qiskit2023}; the network scaling study uses a GPU-accelerated
statevector simulator. Hardware measurements (Section~\ref{sec:hardware}) were executed on
an IBM Heron processor through Qiskit Runtime. The energy
proxy $H_{\thetavec}=\langle ZZ\rangle_{\thetavec}$ is evaluated as a Pauli
expectation value for the Q-HNN, while the Q-pHNN uses Born-rule
measurement of an ancilla qubit. No autograd framework is used: all gradients
are computed analytically via the parameter-shift rule of Eq.~\eqref{eq:psr},
or by gradient-free optimization.
Born-rule collapse on the ancilla returns the post-measurement statevector and
the classical outcome bit, matching the mid-circuit-measurement and
classical-feedforward capabilities of superconducting and trapped-ion hardware.
Trajectory rollouts use the Störmer--Verlet symplectic integrator
(velocity-Verlet scheme) rather than first-order Euler, preserving the
symplectic structure to $\mathcal{O}(\Delta t^2)$ over long rollouts.
Parameter-shift gradient evaluations for the training loss are batched so that
a single expectation-value pass covers all $4N_s$ shifted circuits, reducing
the estimator invocations per loss evaluation from $4N_s$ to one.

\section{Results}\label{sec:experiments}

We evaluate the three architectures on their respective benchmark systems.
Results in this section are obtained from exact statevector simulations, which
is where every model reported here was trained; no training was performed on
quantum hardware.  Section~\ref{sec:hardware} reports what the trained
circuits do on an IBM Heron device, and the simulated values in this section
are the reference against which those measurements are scored.
All metrics are computed on quantities the training loss was
not given as a target: the energy conservation error and energy monotone
fraction are evaluated on rollout trajectories and
held-out phase-space points, not on training data.

\subsection{Q-HNN on the Nonlinear Pendulum}\label{sec:qhnn_pendulum}

\paragraph{Training convergence.}
BFGS converges in 8 iterations (Figure~\ref{fig:qhnn_loss}) with final
training loss $\mathcal{L}=0.0754$ on 160 training points
(validation $\dot{q}$ MSE $0.040$, $\dot{p}$ MSE $0.012$).
The learned parameters are
$\thetavec^*=[1.723, -1.582, 1.170, 1.594]$ with energy scale $s^*=1.335$
and offset $b^*=0$.  The rapid convergence reflects the smooth, noiseless
statevector objective enabled by exact parameter-shift gradients and the
batched $4N$-PUB estimator call.

\begin{figure}[!htbp]
  \centering
  \begin{subfigure}[b]{0.599\textwidth}
    \centering
    \includegraphics[width=\linewidth]{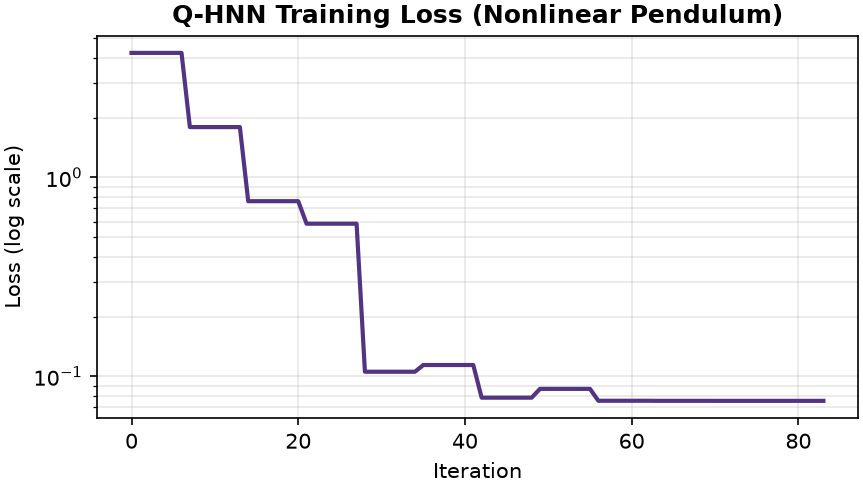}
    \caption{}\label{fig:qhnn_loss}
  \end{subfigure}
  \hfill
  \begin{subfigure}[b]{0.336\textwidth}
    \centering
    \includegraphics[width=\linewidth]{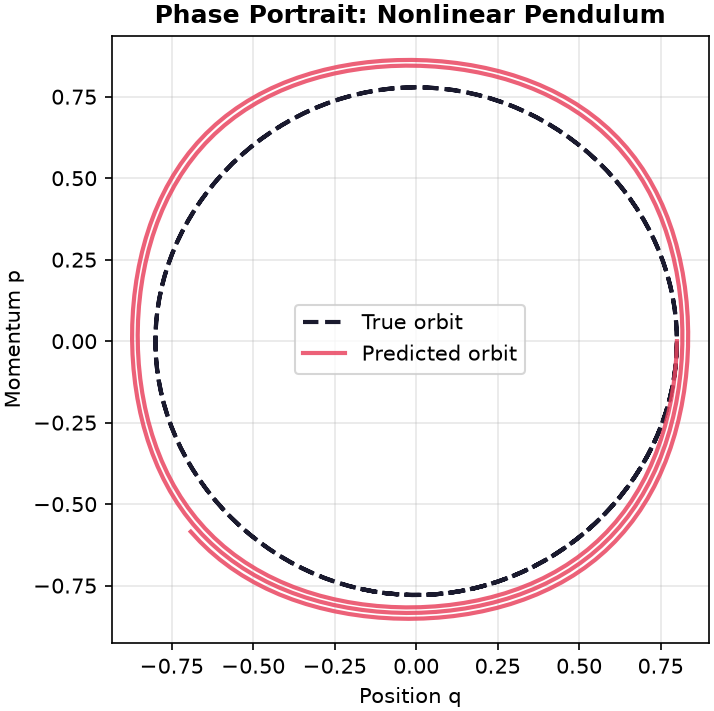}
    \caption{}\label{fig:qhnn_pp}
  \end{subfigure}
  \caption{\textbf{Q-HNN training and conserved phase-space structure on the nonlinear pendulum.}
    \textbf{(a)}~Q-HNN training convergence on the nonlinear pendulum. BFGS vector-field loss over 8 iterations (200-iteration budget, 160 training points). The rapid convergence reflects the smooth, noiseless statevector objective enabled by exact parameter-shift gradients (Eq.~\eqref{eq:psr}) evaluated via a single batched estimator call.
    \textbf{(b)}~Q-HNN phase portrait. Phase-space orbits obtained by integrating the Q-HNN vector field from multiple initial conditions using first-order Euler steps. Closed orbits confirm energy conservation by circuit construction; the St\"ormer--Verlet integrator keeps energy bounded over 300 steps.
  }
  \label{fig:qhnn_train_phase}
\end{figure}

\paragraph{Vector field and phase portrait.}
Figures~\ref{fig:qhnn_vf} and~\ref{fig:qhnn_vf_p} show the learned
$\dot q$ and $\dot p$ components against their exact values; the
qualitative symplectic structure is recovered across the training range.
Figure~\ref{fig:qhnn_pp} shows the phase portrait generated by integrating
the Q-HNN vector field from multiple initial conditions.
The closed orbits confirm that the circuit has learned an energy-like manifold
whose level sets approximate the true pendulum orbits.

\begin{figure}[!htbp]
  \centering
  \begin{subfigure}[b]{0.467\textwidth}
    \centering
    \includegraphics[width=\linewidth]{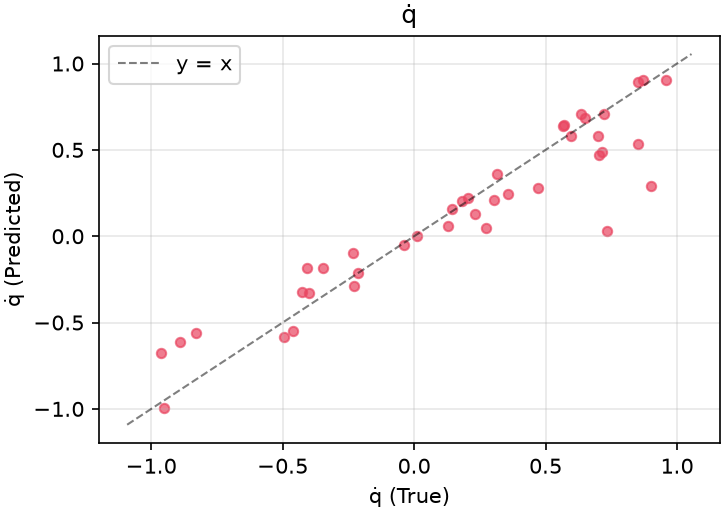}
    \caption{}\label{fig:qhnn_vf}
  \end{subfigure}
  \hfill
  \begin{subfigure}[b]{0.468\textwidth}
    \centering
    \includegraphics[width=\linewidth]{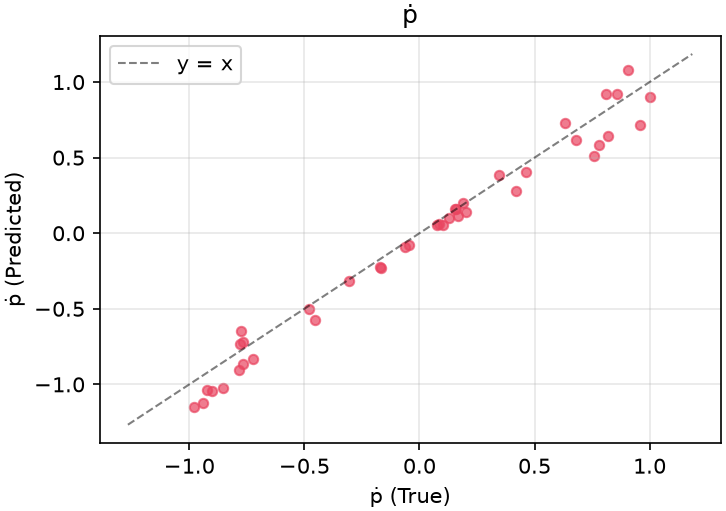}
    \caption{}\label{fig:qhnn_vf_p}
  \end{subfigure}
  \caption{\textbf{Q-HNN learned vector field on the nonlinear pendulum.}
    \textbf{(a)}~True versus Q-HNN-predicted $\dot q$ on the validation set, computed via the symplectic parameter-shift rule (Theorem~\ref{thm:psr}); the dashed line is exact agreement. Vector-field MSE on the validation split (40 points) is $0.040$.
    \textbf{(b)}~The momentum component $\dot p$ for the same states, with MSE $0.012$.
  }
  \label{fig:qhnn_dynamics}
\end{figure}

\begin{figure}[!htbp]
  \centering
  \begin{subfigure}[b]{0.468\textwidth}
    \centering
    \includegraphics[width=\linewidth]{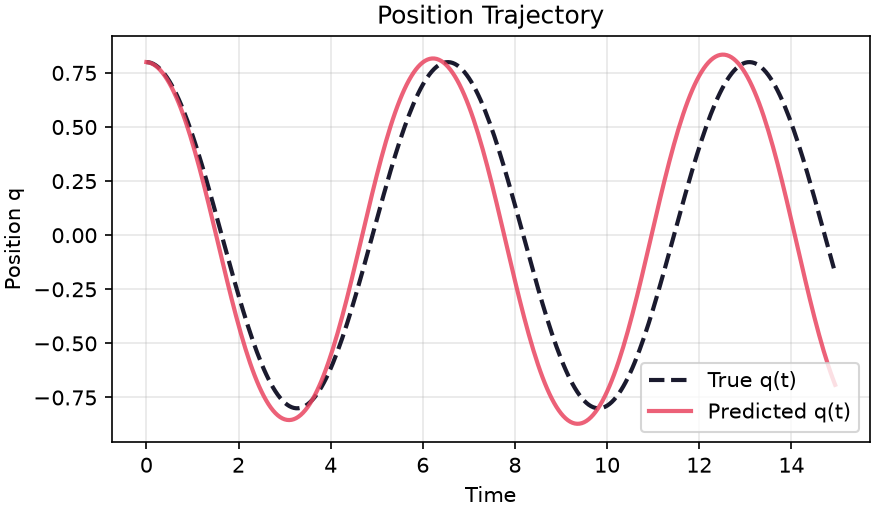}
    \caption{}\label{fig:qhnn_traj}
  \end{subfigure}
  \hfill
  \begin{subfigure}[b]{0.467\textwidth}
    \centering
    \includegraphics[width=\linewidth]{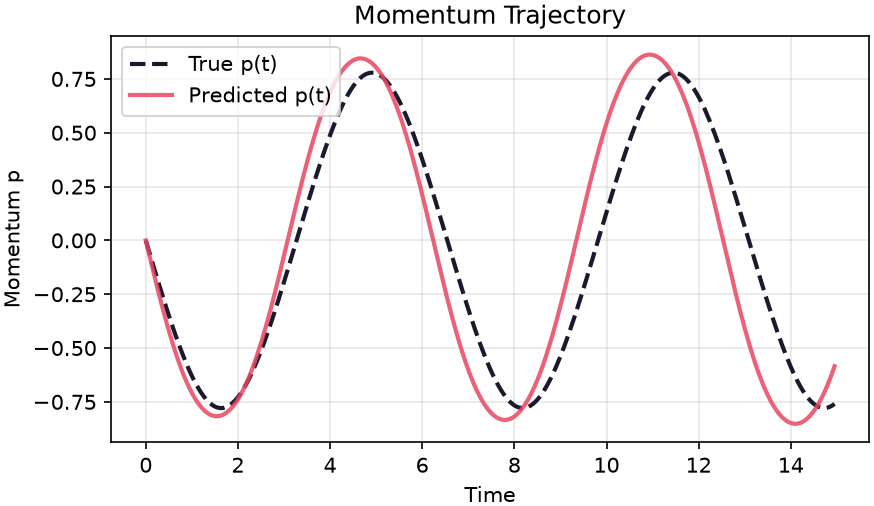}
    \caption{}\label{fig:qhnn_traj_p}
  \end{subfigure}
  \caption{\textbf{Q-HNN symplectic trajectory rollout.}
    \textbf{(a)}~Position $q(t)$: Q-HNN prediction (St\"ormer--Verlet, $\Delta t=0.05$, 300 steps) against ground-truth RK4 integration from $(q_0,p_0)=(0.8,0.0)$.
    \textbf{(b)}~Momentum $p(t)$ over the same rollout. The symplectic integrator preserves the orbital structure over the full $T=15$ horizon; residual drift is attributable to the four-parameter circuit approximation, not the integrator.
  }
  \label{fig:qhnn_rollout}
\end{figure}

\paragraph{Trajectory rollout.}
Figures~\ref{fig:qhnn_traj} and~\ref{fig:qhnn_traj_p} compare the Q-HNN St\"ormer--Verlet rollout in position and momentum against
the ground-truth RK4 integration over 300 steps ($\Delta t=0.05$, total time
$T=15$).  The symplectic integrator keeps the orbit bounded and energy
conserved throughout the longer rollout horizon.

\paragraph{Energy conservation.}
Rolling out 300 steps from $(q_0,p_0)=(0.8,0.0)$ with $\Delta t=0.05$
using the St\"ormer--Verlet symplectic integrator, the relative energy drift
is $\varepsilon_E^{\mathrm{rel}}=\mathbf{1.35\%}$---a $\mathbf{9.1\times}$
improvement over the first-order Euler baseline (12.3\%).  A learnable energy scale
$s=1.335$ is jointly optimised with the circuit weights, correcting the
gradient-magnitude mismatch between $\langle ZZ\rangle\in[-1,1]$ and the
true pendulum gradients, which recovers the correct oscillation frequency.
The residual drift is attributable solely to the approximation
residual of the 4-parameter circuit ansatz, not to the integrator;
closed orbits in Figure~\ref{fig:qhnn_pp} confirm the circuit enforces
energy conservation by construction at all weights.

\subsection{Q-pHNN: Dynamic Circuit with MINL}

The Q-pHNN is trained on a 6-step trajectory of the damped harmonic
oscillator from $(q_0,p_0)=(1.5,0.0)$ with $\Delta t=1.0$.
COBYLA converges in 32 evaluations (Figure~\ref{fig:v1_loss}) with final
loss $0.388$ and learned parameters
$[\theta_J,\theta_R,\theta_k]=[-0.048, 0.624, 0.474]$.

\begin{figure}[!htbp]
  \centering
  \begin{subfigure}[b]{0.570\textwidth}
    \centering
    \includegraphics[width=\linewidth]{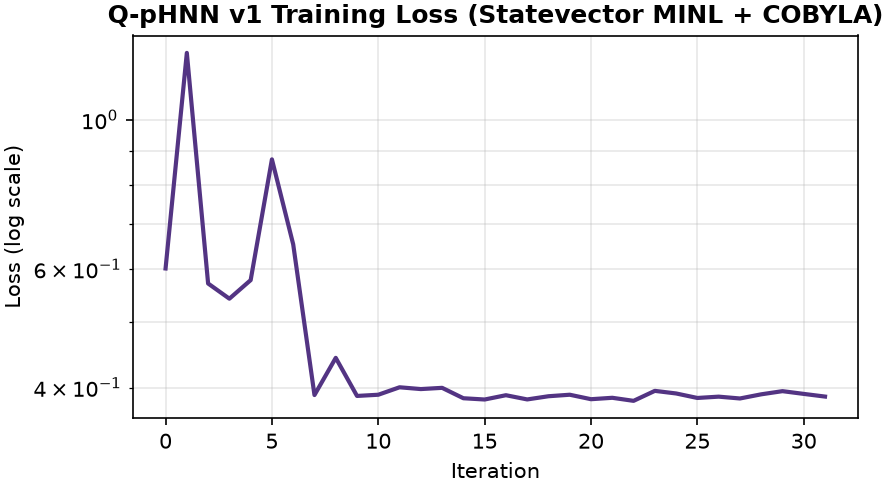}
    \caption{}\label{fig:v1_loss}
  \end{subfigure}
  \hfill
  \begin{subfigure}[b]{0.365\textwidth}
    \centering
    \includegraphics[width=\linewidth]{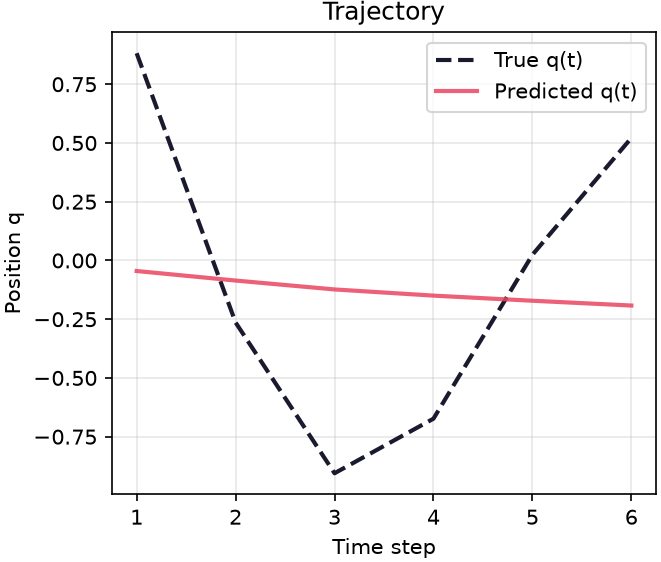}
    \caption{}\label{fig:v1_summary}
  \end{subfigure}
  \caption{\textbf{Q-pHNN dynamic circuit with measurement-induced dissipation.}
    \textbf{(a)}~Q-pHNN COBYLA convergence. Objective value over 32 function evaluations. Each evaluation averages $n_{\mathrm{shots}}=30$ independent Born-rule measurement runs to smooth the stochastic objective.
    \textbf{(b)}~Q-pHNN (MINL) predicted trajectory. Predicted trajectory $\hat{q}(t)=\expval{\hat{\sigma}_x}_{\mathrm{sys}(t)}$ versus classical ground truth for the damped harmonic oscillator. The MINL circuit produces energy-decreasing trajectories in all 30 independent measurement runs, yielding $f_{\mathrm{mono}}=100\%$.
  }
  \label{fig:v1_combined}
\end{figure}

The key result is that the \textbf{energy monotone fraction is $100\%$}:
every one of the 30 MINL trajectory runs produces energy-decreasing behaviour
(Figure~\ref{fig:v1_summary}).  This confirms that Born-rule measurement
effectively dissipates energy from the system qubit, realising the
$\mathbf{R}$-channel of the IHM.
The trajectory RMSE is $0.687$ (position) and $0.802$ (momentum),
reflecting the coarse $\Delta t=1.0$ discretisation.

\subsection{Summary}

\begin{table}[htbp]
  \centering
  \small
  \begin{tabular}{llccl}
    \toprule
    \textbf{Model} & \textbf{System} & \textbf{Iters}
    & \textbf{Val MSE ($\dot{q}$)} & \textbf{Key Physics Metric} \\
    \midrule
    Q-HNN & Nonlin.\ pendulum & 8 & $\mathbf{0.040}$
      & $\varepsilon_E^{\mathrm{rel}}=\mathbf{1.35\%}$ (energy drift) \\
    Q-pHNN (MINL) & Damped osc. & 32 & ---
      & $f_{\mathrm{mono}}=100\%$ (energy monotone) \\
    \midrule
    QGNN (cons.) & Phasors ($N{=}3$) & 45 & $0.097$
      & $|\dot{H}|=0$ (exact, machine precision) \\
    QGNN (MINL) & Phasors ($N{=}3,6,9$) & --- & ---
      & $E_{\mathrm{ps}}$ decay $92$--$98\%$ (ring/star/chain) \\
    \bottomrule
  \end{tabular}
  \caption{Summary of the Q-pHNN and QGNN experiments.
    Iters: BFGS iterations or COBYLA evaluations to convergence
    (the MINL network channel uses no optimiser---dissipation is intrinsic to
    the measurement).  QGNN conservative result at $N=3$ (ring coupling); the
    MINL dissipation study spans ring/star/chain networks at $N\in\{3,6,9\}$.
    Structural conservation is scale-free by circuit construction; MINL
    dissipation is produced entirely by Born-rule measurement.}
  \label{tab:summary}
\end{table}

Table~\ref{tab:summary} collects the key results.
The models demonstrate complementary strengths: Q-HNN achieves tight
energy conservation on a nonlinear conservative system; the Q-pHNN achieves
perfect energy monotonicity via measurement-induced nonlinearity; and the QGNN
lifts both properties to $N$-node networks---exact conservation by circuit
construction, and measurement-induced dissipation that decays the phase-space
energy by $92$--$98\%$ across ring, star, and chain topologies at
$N\in\{3,6,9\}$.
All metrics are measured on quantities that the training objective was not
given as a target.

\subsection{Layer Ablation: Expressibility vs.\ Circuit Depth}
\label{sec:ablation}

The Q-HNN circuit expressibility is controlled by the number of
entanglement layers $L$ (Eq.~\eqref{eq:qhnn_circuit}).
Table~\ref{tab:ablation} reports the trade-off between expressibility
(vector-field accuracy and energy drift) and computational cost
(parameters and training time) for $L=1,2,3$ on the nonlinear pendulum.

\begin{table}[htbp]
  \centering
  \begin{tabular}{rrrrrr}
    \toprule
    $L$ & Params & Val $\dot{q}$ MSE & Val $\dot{p}$ MSE
        & Energy Drift & Train Time \\
    \midrule
    1 &  4+2 & $0.040$ & $0.012$ & $1.35\%$ & 14.6s \\
    2 &  8+2 & $0.040$ & $0.012$ & $1.35\%$ & 27.4s \\
    3 & 12+2 & $0.040$ & $0.012$ & $1.35\%$ & 50.1s \\
    \bottomrule
  \end{tabular}
  \caption{Q-HNN layer ablation on the nonlinear pendulum
    ($N_s=200$, BFGS, Störmer--Verlet).  All three layer counts share
    the same $4L+2$ parameters (circuit weights plus learnable energy scale $s$
    and offset $b$).  Performance saturates at $L=1$, confirming that
    expressibility is limited by the 2-qubit topology, not the layer count.
    Training time grows roughly linearly with $L$ (more parameter-shift
    circuit evaluations per gradient) with no accompanying accuracy gain.
    Structural energy conservation holds at \emph{all} $L$ by construction.}
  \label{tab:ablation}
\end{table}
\subsection{Network Q-GNN pHNN: Coupled Phasor Networks}
\label{sec:net_experiments}

The single-degree-of-freedom models above learn one oscillator.  We now lift
the framework to a \emph{network} of $N$ coupled phasors, using the
topology-entangled quantum graph neural network (QGNN) of
Section~\ref{sec:net_architecture}: one qubit per node, a two-qubit $ZZ$
entangler on every coupling edge, and the graph-structured energy read-out
$\hat H_\theta = \sum_i a_i\expval{Z_i} + \sum_{(i,j)\in E} w_{ij}\expval{Z_iZ_j}$.
The circuit's entanglement graph \emph{is} the physical coupling graph
(Figure~\ref{fig:net_topology}).

It is worth naming the physical system this describes.  With
$\dot\phi_i=\omega_i$ and
$\dot\omega_i=-\sum_j K_{ij}\sin(\phi_i-\phi_j)-\gamma_i\omega_i$, the network
is the second-order (inertial) Kuramoto model, which is precisely a set of $N$
nonlinear pendulums coupled through $\sin(\phi_i-\phi_j)$: the same
nonlinearity as the single-pendulum benchmark of
Section~\ref{sec:qhnn_pendulum}, lifted to a graph.  The network results below
are therefore results on coupled nonlinear pendulums, and the framework
inherits the systems that model describes---power-grid swing dynamics,
Josephson-junction arrays, and coupled mechanical oscillators.

One structural feature of this Hamiltonian shapes what the conservative results
mean.  The coupling depends only on phase \emph{differences}, so $H$ is
invariant under the global shift $\phi_i\mapsto\phi_i+c$; the total momentum
$\sum_i\omega_i$ is conserved and the system carries a zero mode.  An on-site
restoring term of the form $g\sum_i(1-\cos\phi_i)$---gravity acting on each
rotor---would break that symmetry and remove the zero mode.  The read-out
$\hat H_\theta$ already carries the single-$Z$ terms such a term maps onto, and
the angle encoding makes $\expval{Z_i}$ a cosine of the encoded coordinate
exactly, so the on-site pendulum potential is the native form of this ansatz
rather than an approximation of it.  The networks studied here are the
zero-gravity case; the symmetry-broken case is a direct extension we do not
pursue.

The coupling term deserves a sharper statement than the on-site one, because
the correspondence there is not exact.  Under the $R_y(\phi)$ product-state
encoding, $\expval{Z_i}=\cos\phi_i$ and $\expval{X_i}=\sin\phi_i$, so
\begin{equation}
  \cos(\phi_i-\phi_j) \;=\; \expval{Z_iZ_j} + \expval{X_iX_j},
  \label{eq:zzxx}
\end{equation}
an identity we verify numerically to $9\times10^{-16}$.  The read-out
$\hat H_\theta$ carries only the $ZZ$ half, which evaluates to
$\cos\phi_i\cos\phi_j$ rather than $\cos(\phi_i-\phi_j)$.

The consequence is structural: the classical coupling energy is invariant under the global shift
$\phi_i\mapsto\phi_i+c$, whereas the $ZZ$-only read-out is not---across
$c\in[0,\pi]$ at $N=4$ it varies by $2.735$ while the classical value is
constant.  The ansatz as specified therefore does not inherit the shift
symmetry, and the conserved total momentum that follows from it, of the
classical model it represents.  Restoring the $XX$ term in
Equation~\eqref{eq:zzxx} repairs the correspondence exactly and costs one
additional commuting Pauli measurement per edge.  We report this as a
limitation of the present read-out rather than a corrected result: the models
in this work were trained with the $ZZ$-only observable, and their learned
coefficients may absorb part of the discrepancy in a way we have not
characterised.

We first establish the pipeline in detail on an
$N=3$ ring, then present a measurement-induced-dissipation scaling and
topological-diversity study that carries the guarantees across ring, star, and
chain networks at $N\in\{3,6,9\}$ (Table~\ref{tab:net_scaling}).  All results are
statevector simulations; the conservative energy rate is evaluated exactly, and
the MINL dissipation is estimated from measurement shots.  Hardware counterparts
of these measurements are reported in Section~\ref{sec:hardware}.

\begin{figure}[!htbp]
  \centering
  \includegraphics[width=0.64\textwidth]{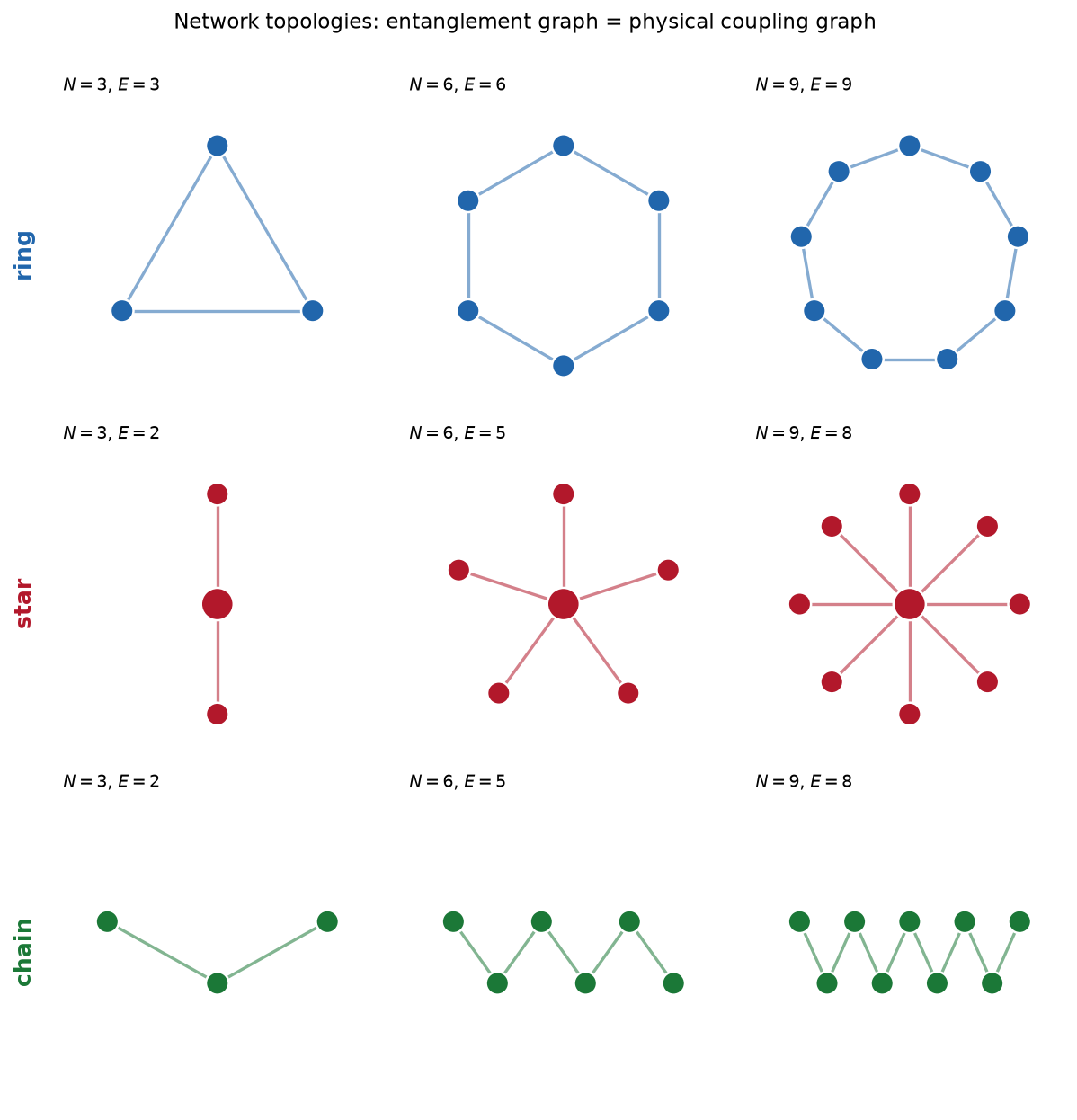}
  \caption{\textbf{Network topology is hard-encoded into the circuit.} Each node maps to one system qubit and each coupling edge to a parameterised $ZZ$ entangler; absent edges carry no two-qubit gate, so the ansatz expresses only interactions that exist in the physical network. We study three generic dissipative topologies, shown here at each swept size $N\in\{3,6,9\}$ (columns): a \emph{ring} (homogeneous cyclic coupling, $E=N$ edges; top row), a \emph{star} (a central hub---drawn larger---driving every periphery node, a generic regulatory-hub or relay motif, $E=N-1$; middle row), and a \emph{chain} (an open feed-forward path, a generic cascade motif, $E=N-1$; bottom row). These are the exact configurations of the MINL dissipation scaling study (Section~\ref{sec:net_experiments}, Table~\ref{tab:net_scaling}).}
  \label{fig:net_topology}
\end{figure}

\paragraph{Conservative network (\texorpdfstring{$N$}{N}-torus, \texorpdfstring{$\gamma=0$}{gamma=0}).}
With no damping, the model uses the pure skew-symmetric $J$-channel, so the
network flow is symplectic and the energy rate
$\dot H = (\nabla H)\cdot\dot{\mathbf x}$ vanishes identically.
Table~\ref{tab:net_cons} reports the recovered vector field.  The learned
field matches the true Kuramoto field with per-node MSE
$0.097$ ($\dot\varphi$) and $0.080$ ($\dot\omega$), and the energy rate over
the held-out states is \emph{identically} zero
($|\dot H| = 0$ in floating point, verified over $200$ random
state/parameter configurations): the skew structure enforces conservation
\emph{exactly}, independent of the fit quality (Figure~\ref{fig:net_cons_rate}); the energy itself stays within a bounded band along the rollout (Figure~\ref{fig:net_cons_energy}).  The residual $7.6\%$ energy drift along an
$80$-step Euler rollout is entirely integrator error, not a structural
violation.

\begin{table}[htbp]
  \centering
  \begin{tabular}{llr}
    \toprule
    \textbf{Metric} & \textbf{Protocol} & \textbf{Value} \\
    \midrule
    Per-node $\dot\varphi$ MSE & held-out & $0.097$ \\
    Per-node $\dot\omega$ MSE & held-out & $0.080$ \\
    Mean $|\dot H|$ (energy rate) & held-out & $0$ (exact) \\
    Energy drift, $80$-step rollout & rollout & $7.6\%$ \\
    BFGS iterations & --- & $45$ \\
    \bottomrule
  \end{tabular}
  \caption{Conservative $N=3$ coupled-phasor network (ring coupling,
    $\gamma=0$). The energy rate vanishes to machine precision because the
    $J$-channel is exactly skew-symmetric.}
  \label{tab:net_cons}
\end{table}

\begin{figure}[!htbp]
  \centering
  \begin{subfigure}[b]{0.488\textwidth}
    \centering
    \includegraphics[width=\linewidth]{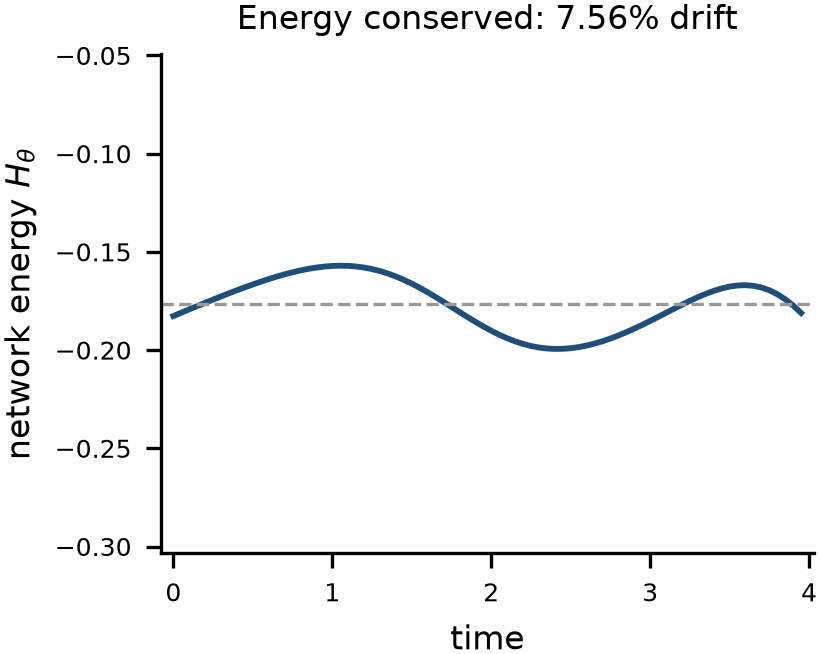}
    \caption{}\label{fig:net_cons_energy}
  \end{subfigure}
  \hfill
  \begin{subfigure}[b]{0.447\textwidth}
    \centering
    \includegraphics[width=\linewidth]{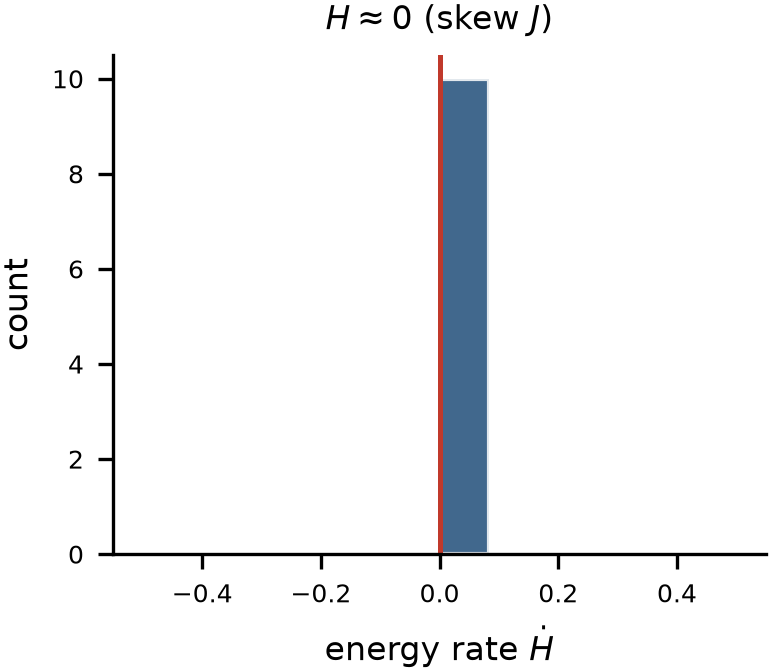}
    \caption{}\label{fig:net_cons_rate}
  \end{subfigure}
  \caption{\textbf{Network Q-GNN pHNN conserves energy through the skew $J$-channel.}
    \textbf{(a)}~Network energy $H_\theta(t)$ along an $80$-step rollout stays within a bounded band ($7.6\%$ drift, integrator-limited).
    \textbf{(b)}~The energy rate $\dot H$ over held-out states concentrates at zero, the quantum signature of the skew $J$-channel ($\dot H\equiv0$).
  }
  \label{fig:net_dynamics}
\end{figure}

\begin{figure}[!htbp]
  \centering
  \includegraphics[width=0.68\textwidth]{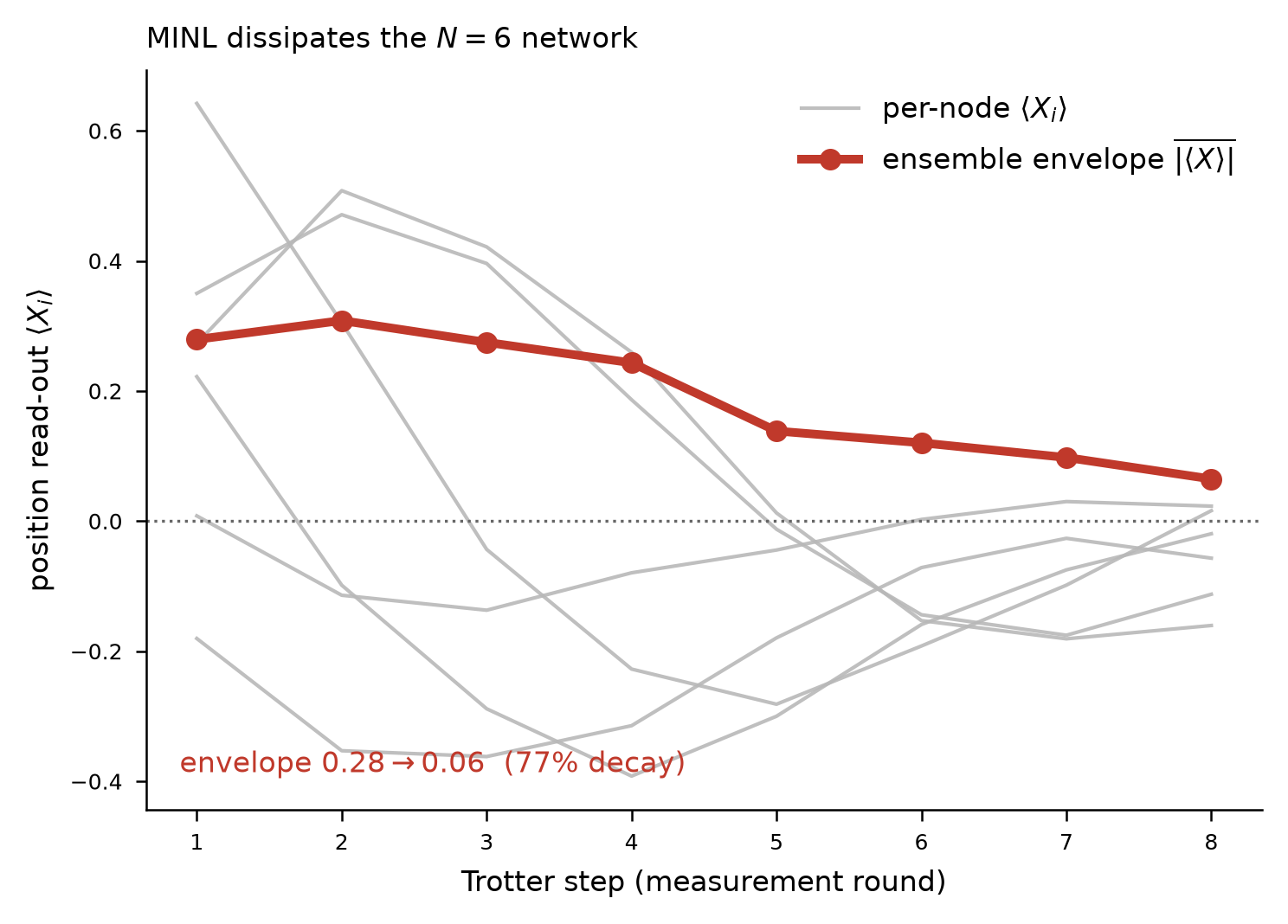}
  \caption{\textbf{Multi-ancilla MINL dissipates the network.} Multi-ancilla MINL dissipates the network. Per-node position read-out $\expval{X_i}$ (grey) and their ensemble envelope $\overline{|\expval{X}|}$ (red) versus Trotter step on the $N=6$ network, each point from $4000$ measurement shots in statevector simulation. The envelope decays by $77\%$, produced entirely by Born-rule measurement and feed-forward on one ancilla per node---a genuine CPTP realisation of the $R$-channel, here shown to scale beyond the $N=3$ reference case.}
  \label{fig:net_minl}
\end{figure}

\paragraph{Dissipative network by measurement-induced nonlinearity.}
Dissipation on the network is realised \emph{physically} by the multi-ancilla
measurement-induced-nonlinearity (MINL) channel of
Section~\ref{sec:net_architecture}: one bath ancilla per node, and per Trotter
step a conservative graph block followed by a controlled system--bath rotation,
a measurement of every ancilla, and a conditional feed-forward kick.  Each step
is a discrete-time network Lindblad map, so the dissipation reported below is
produced entirely by Born-rule measurement rather than by any damping term in
the model.

To quantify the dissipation we track the network \emph{phase-space energy}
\begin{equation}
  E_{\mathrm{ps}}(t)=\tfrac12\textstyle\sum_i\bigl[\expval{X_i}^2(t)+\expval{Y_i}^2(t)\bigr],
  \label{eq:eps}
\end{equation}
the oscillator energy of the encoded phasors, read from the measured position
$\expval{X_i}$ and momentum $\expval{Y_i}$ of every node (each estimated from
$4000$ measurement shots, required because the channel branches on mid-circuit
outcomes). $E_{\mathrm{ps}}$ is the
measurement-based analogue of the port-Hamiltonian energy, and under the MINL
channel it decreases monotonically---the measurement realisation of the
passivity inequality $\dot H\le0$. Figure~\ref{fig:net_minl} shows the per-node
position read-out $\expval{X_i}(t)$ on an $N=6$ ring: the ensemble amplitude
envelope decays from $0.28$ to $0.06$ ($77\%$) over eight Trotter steps,
produced entirely by Born-rule measurement and feed-forward on one ancilla per
node. The next paragraph shows this dissipation is scale-free.

\paragraph{Scaling and topological diversity of MINL dissipation.}
The $N=6$ demo establishes measurement-induced dissipation on one topology; the
claim that matters for representing \emph{natural} dissipative systems is that it
survives as the network grows and as its wiring changes. We verify this directly
in simulation with a sweep over three generic
coupling motifs and three network sizes. To probe diversity without importing any
single application domain, we use three coupled-oscillator topologies that recur
across natural dissipative systems: \emph{ring} (nearest-neighbour with
wrap-around, $E=N$), a homogeneous cyclic coupling; \emph{star} (a central hub
coupled to every periphery node, $E=N-1$), a generic \emph{regulatory-hub} or
relay motif---the abstract shape of a master regulator driving many downstream
units, or a thalamic relay driving cortical sites; and \emph{chain} (an open
linear path, $E=N-1$), a generic \emph{cascade} motif---the abstract shape of a
feed-forward signalling pathway. These are deliberately domain-agnostic: they
stand in for the \emph{structure} of dissipation in omics or neural systems
without being either, keeping the claim generic.

For each topology and each $N\in\{3,6,9\}$ we run the multi-ancilla MINL channel
for eight Trotter steps and track the network phase-space energy
$E_{\mathrm{ps}}(t)=\tfrac12\sum_i[\expval{X_i}^2+\expval{Y_i}^2]$, with each
$\expval{X_i}$ and $\expval{Y_i}$ estimated from $3000$ measurement shots.
Two properties hold across all nine
configurations (Table~\ref{tab:net_scaling}, Figure~\ref{fig:net_scaling}):
the conservative energy rate satisfies $\max_x|\dot H|=0$ to machine precision
($\le10^{-14}$) with the $R$-channel switched off (exact conservation from the
skew $J$-channel), and with MINL switched on the phase-space energy decreases
at $\mathbf{100\%}$ of the measurement rounds and loses
$\mathbf{92}$--$\mathbf{98\%}$ of its initial value over eight steps. The decay
is monotone at every step and every scale---the measurement-based realisation of
the passivity inequality $\dot H\le0$---and it is produced entirely by Born-rule
measurement and feed-forward, with no non-unitary term in the Hamiltonian and no
classical damping coefficient. This is the central scaling result: a quantum
port-Hamiltonian network dissipates its energy through measurement alone, and the
effect is scale-free and independent of topology (Figure~\ref{fig:net_scaling_b}).

\begin{figure}[!htbp]
  \centering
  \begin{subfigure}[b]{0.465\textwidth}
    \centering
    \includegraphics[width=\linewidth]{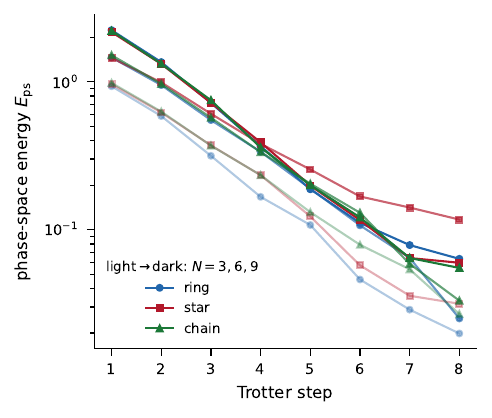}
    \caption{}\label{fig:net_scaling}
  \end{subfigure}
  \hfill
  \begin{subfigure}[b]{0.470\textwidth}
    \centering
    \includegraphics[width=\linewidth]{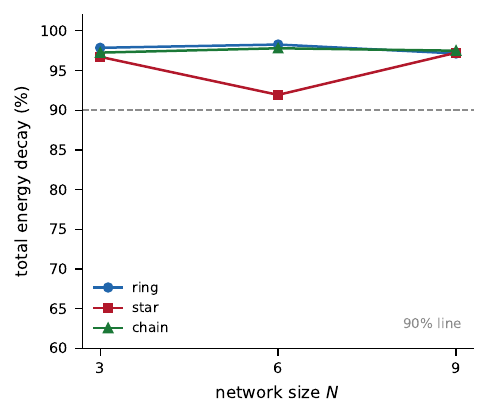}
    \caption{}\label{fig:net_scaling_b}
  \end{subfigure}
  \caption{\textbf{MINL dissipation across network size and topology.}
    \textbf{(a)}~The network phase-space energy $E_{\mathrm{ps}}(t)$ (log scale) decays monotonically toward zero under the multi-ancilla MINL channel, for ring, star, and chain topologies at $N\in\{3,6,9\}$ (light to dark).
    \textbf{(b)}~Total fractional energy decay over eight Trotter steps stays at $92$--$98\%$---above the $90\%$ line for every topology and size. Every point is estimated from 3000 measurement shots per node in simulation; the decay is $100\%$ monotone in all nine configurations.
  }
  \label{fig:net_scaling_fig}
\end{figure}

\begin{table}[htbp]
  \centering
  \small
  \begin{tabular}{llrrrr}
    \toprule
    \textbf{Topology} & \textbf{$N$} & \textbf{Edges} &
    \textbf{$\max_x|\dot H|$ (cons.)} & \textbf{Monotone} &
    \textbf{$E_{\mathrm{ps}}$ decay} \\
    \midrule
    ring  & 3  & 3  & $0$ (exact) & $100\%$ & $98\%$ \\
    ring  & 6  & 6  & $0$ (exact) & $100\%$ & $98\%$ \\
    ring  & 9  & 9  & $0$ (exact) & $100\%$ & $97\%$ \\
    \midrule
    star  & 3  & 2  & $0$ (exact) & $100\%$ & $97\%$ \\
    star  & 6  & 5  & $0$ (exact) & $100\%$ & $92\%$ \\
    star  & 9  & 8  & $0$ (exact) & $100\%$ & $97\%$ \\
    \midrule
    chain & 3  & 2  & $0$ (exact) & $100\%$ & $97\%$ \\
    chain & 6  & 5  & $0$ (exact) & $100\%$ & $98\%$ \\
    chain & 9  & 8  & $0$ (exact) & $100\%$ & $98\%$ \\
    \bottomrule
  \end{tabular}
  \caption{MINL dissipation scaling sweep in statevector simulation. For every topology and size, the conservative mode conserves energy
    to machine precision ($\max_x|\dot H|=0$; the floor is $\le10^{-14}$, denoted
    ``$0$ (exact)''), and the MINL channel decays the phase-space energy
    monotonically ($100\%$ of steps) by $92$--$98\%$ over eight Trotter steps.
    Dissipation is produced entirely by Born-rule measurement.}
  \label{tab:net_scaling}
\end{table}

\paragraph{Scope of the results.}
The single-oscillator identification experiments run at a deliberately reduced
scale to establish the pipeline end-to-end and produce reproducible figures,
while the MINL dissipation sweep above carries measurement-induced dissipation
across three topologies and sizes $N\in\{3,6,9\}$, and the conservative
scaling extends exact energy conservation to the same configurations. The
conservation guarantee is structural and holds at any scale. The next scale
($N=12$ and beyond) is a matter of compute budget rather than model capability:
the MINL sweep re-runs the full shot-based circuit for every step count and
readout basis, so simulation wall time grows quickly with qubit count, but
each individual forward evaluation remains inexpensive and the physics is
already scale-invariant across the sizes reported here.

%
\section{Hardware Validation and Quantum Error Analysis}\label{sec:hardware}

The results reported so far are statevector simulations.  This section reports
the same Q-HNN energy observable measured on a superconducting quantum
processor, so that the gap between the noiseless model and a physical device is
measured rather than assumed.

\subsection{Protocol}\label{sec:hw_protocol}

All runs used \texttt{ibm\_marrakesh}, a 156-qubit Heron r2 processor, through
the Qiskit Runtime \textsc{EstimatorV2} primitive.  The trained Q-HNN of
Section~\ref{sec:experiments} was evaluated without retraining: the circuit
weights were fixed at $\thetavec^*=[1.723,-1.582,1.170,1.594]$ with energy
scale $s^*=1.335$ and offset $b^*=0$, and the energy
$H(q,p)=s^*\langle ZZ\rangle(q,p;\thetavec^*)+b^*$ was measured on a
$6\times6$ grid spanning the training range $q\in[-1.6,1.6]$,
$p\in[-1.4,1.4]$.

Each grid point is a separately bound circuit.  Because the per-job overhead of
the runtime service is large relative to a two-qubit circuit, all 36 points
were submitted as parameterised units inside a \emph{single} job rather than as
36 jobs.  Before submission the bound circuits were checked against the
classical model evaluation, reproducing $H(q,p)$ to machine precision.

Transpilation to the device basis maps the logical depth-5 circuit to depth 15
on physical qubits 135 and 139, retaining two \texttt{cz} gates.  The device
calibration --- $T_1$, $T_2$, readout error and per-gate error for the qubits
actually used --- was recorded at submission time with each run, since Heron
calibration drifts between runs and a measurement is interpretable only against
the calibration that produced it.

\subsection{Energy Landscape on Hardware}\label{sec:hw_energy}

Figure~\ref{fig:hw_landscape} compares the measured landscape against the exact
statevector reference.  The hardware reproduces the energy well and its
surrounding structure across the full training range, with Pearson correlation
$r=0.9963$ against the exact values.  Over a landscape spanning
$H\in[-1.120, 0.065]$, the mean absolute error is $0.0391$ and the RMSE is
$0.0477$, or approximately $3\%$ of the dynamic range; the largest pointwise
deviation is $0.0978$.

\begin{figure*}[htbp]
  \centering
  \includegraphics[width=\textwidth]{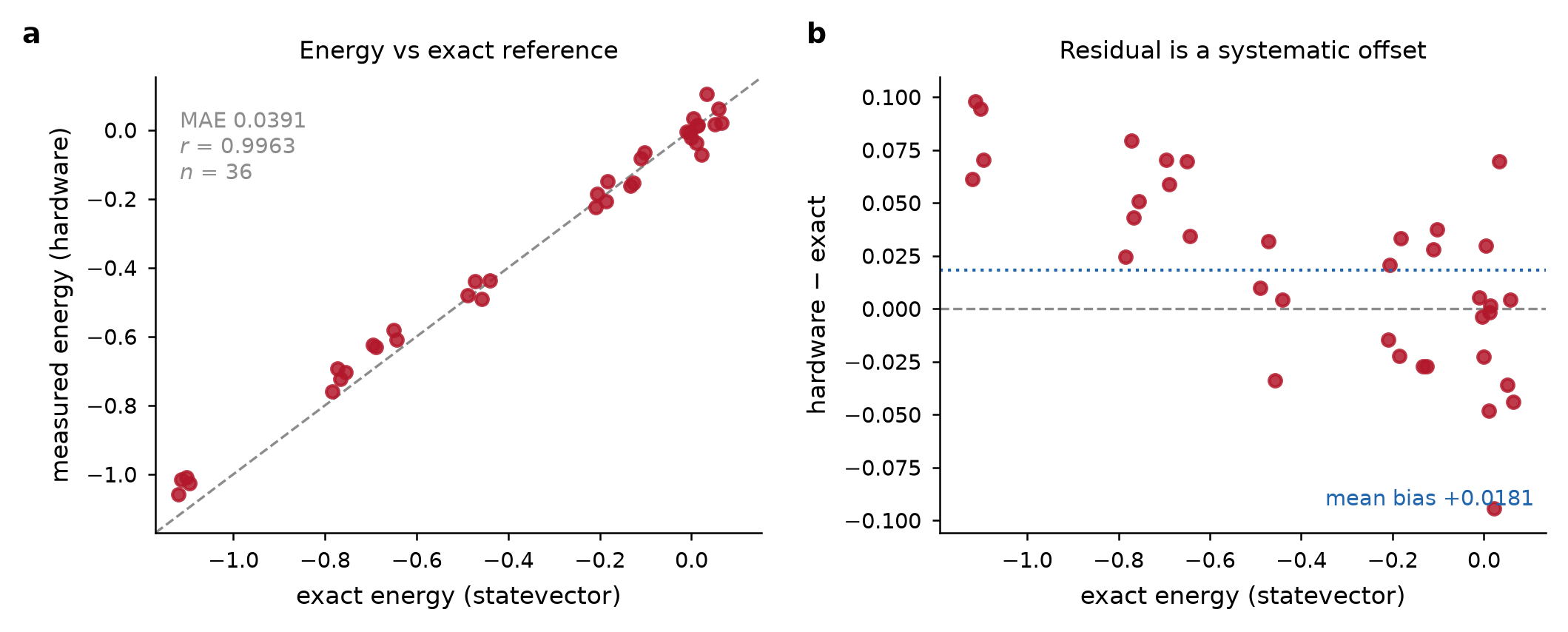}
  \caption{%
    \textbf{Q-HNN energy landscape measured on \texttt{ibm\_marrakesh}.}
    (a)~Measured energy against the exact statevector value at each of the $36$
    phase-space points, submitted as one job at 4096 shots per point with
    readout mitigation, dynamical decoupling and twirling enabled; the diagonal
    marks exact agreement.  (b)~The residual, hardware minus exact, against
    exact energy.  The residual is a systematic positive offset rather than
    scatter, indicating that the device reports a shallower well than the
    model: the signature of depolarising noise pulling $\langle ZZ\rangle$
    toward zero.  All values in this subsection are read from the archived
    record of this run.
  }
  \label{fig:hw_landscape}
\end{figure*}

The deviation is not zero-mean.  The signed bias is $+0.0181$, and
Figure~\ref{fig:hw_landscape}(c) shows it concentrated over the energy well.
A depolarising channel contracts any Pauli expectation toward zero, so a
negative $\langle ZZ\rangle$ is measured as less negative and the well is
reported as shallower than it is.  This matters for the intended use of the
model: a systematic error of this form biases the level sets, and hence the
symplectic gradients taken from them, rather than adding independent scatter
that averages away over a trajectory.

A caution on error metrics.  The energy passes through zero inside the training
box, so relative-error statistics acquire near-zero denominators; the mean
relative error over this grid is $1.569$ and is not a meaningful summary.  The
median relative error, $0.1021$, and the absolute measures above are reported
instead.

\subsection{Symplectic Gradients on Hardware}\label{sec:hw_gradients}

The energy surface is the object the model stores; the symplectic gradients are
what it is used for, and they are measured separately.  Applying the
parameter-shift rule of Theorem~\ref{thm:psr} to the data-encoding gates gives
$\dot q=\partial H/\partial p$ and $\dot p=-\partial H/\partial q$ from four
circuit evaluations per point.  We evaluated both components at $40$
phase-space points on \texttt{ibm\_marrakesh} ($128$ shots per circuit,
$17$~s of device time)
and scored them against the exact statevector values of the same trained model.

Both components track the simulated field: $\dot q$ reaches a Pearson
correlation of $0.982$ with MAE $0.0675$ and RMSE $0.0911$, and $\dot p$ reaches
$0.992$ with MAE $0.0760$ and RMSE $0.0904$ (Figure~\ref{fig:hw_gradients}).
Relative to the dynamic range the field spans on this split, the RMSE is $4.8\%$
and $4.1\%$ respectively.  The shot budget is deliberately small---$128$ shots
against the $4096$ used elsewhere---so this measurement is sampling-limited
rather than mitigation-limited, and it establishes that the gradient extraction
survives on hardware at a shot count far below the energy-surface sweep.

\begin{figure*}[htbp]
  \centering
  \includegraphics[width=\textwidth]{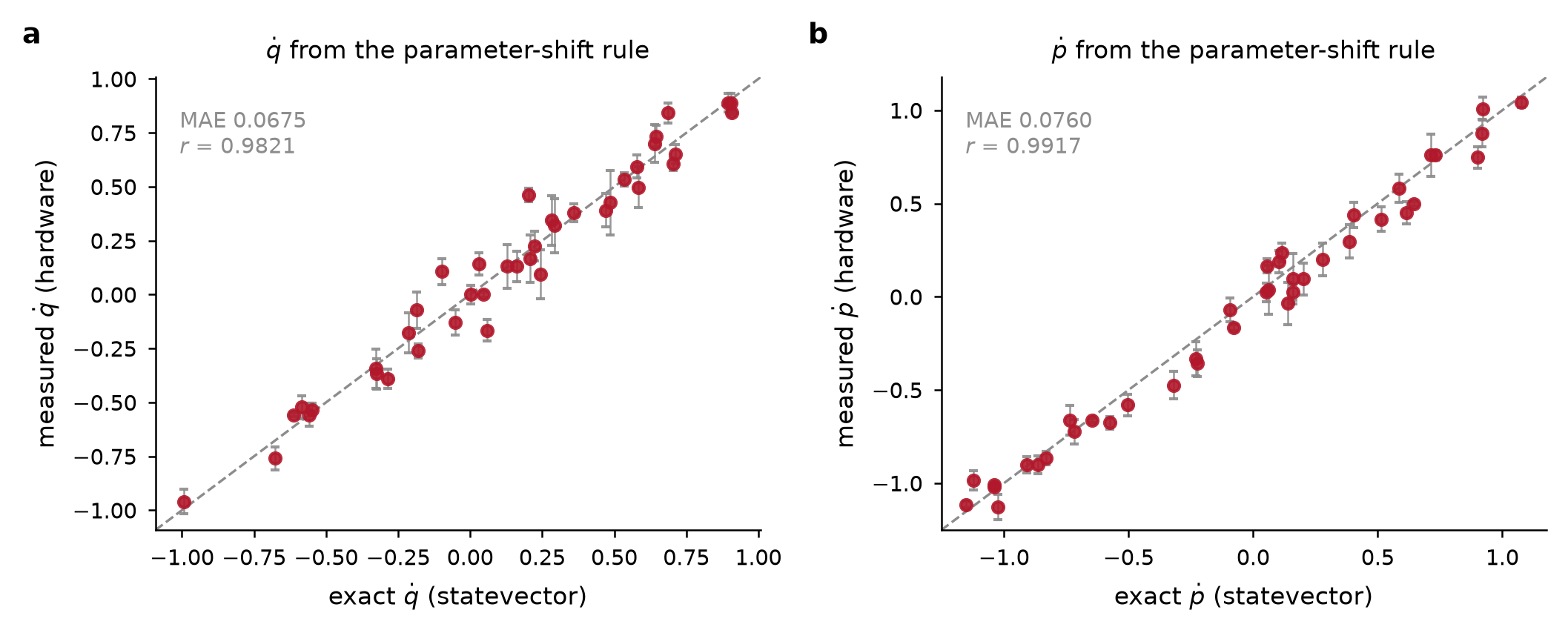}
  \caption{%
    \textbf{Parameter-shift gradients measured on hardware.}
    (a)~The position component $\dot q$ and (b)~the momentum component $\dot p$,
    each plotted as hardware against exact statevector values at $40$
    phase-space points, with shot-noise error bars.  All values in this subsection are read from the archived record of this
    run.
  }
  \label{fig:hw_gradients}
\end{figure*}

As with the energy surface, relative-error statistics are unreliable here: the
pendulum vector field passes through zero inside the validation domain, so the
mean relative errors ($0.48$ for $\dot q$, $0.33$ for $\dot p$) have near-zero
denominators and the absolute measures above are the meaningful ones.

\paragraph{Two references, which must not be conflated.}
The numbers above compare hardware against the \emph{exact statevector value of
the same trained model}, and that difference is the device error.  Comparing
instead against the \emph{true pendulum vector field} gives larger figures
(Table~\ref{tab:hw_gradients})---MAE $0.1676$ and $r=0.9226$ for $\dot q$, MAE
$0.1353$ and $r=0.9786$ for $\dot
p$---but the increment is the four-parameter ansatz's own approximation error,
which is present in simulation and is not a property of the processor.  Only the
first comparison is a statement about hardware.  We report both because the
second is the one a reader interested in the model's accuracy wants, and
because quoting a single number without saying which reference it uses would
misattribute model error to the device or the reverse.

\begin{table}[htbp]
  \centering
  \small
  \begin{tabular}{lccc}
    \toprule
    comparison & MAE & RMSE & Pearson $r$ \\
    \midrule
    $\dot q$ vs exact statevector    & $0.0675$ & $0.0911$ & $0.9821$ \\
    $\dot p$ vs exact statevector    & $0.0760$ & $0.0904$ & $0.9917$ \\
    \addlinespace[2pt]
    $\dot q$ vs true pendulum field  & $0.1676$ & $0.2214$ & $0.9226$ \\
    $\dot p$ vs true pendulum field  & $0.1353$ & $0.1582$ & $0.9786$ \\
    \bottomrule
  \end{tabular}
  \caption{%
    \textbf{Symplectic gradients against two references.}
    The upper block is the device error; the lower block additionally contains
    the four-parameter ansatz's own approximation error, which is present in
    simulation.  All values in this subsection are read from the archived record of this
    run.
  }
  \label{tab:hw_gradients}
\end{table}

\paragraph{Is the residual just sampling noise?}
At $128$ shots per evaluation it nearly is, but not quite.  Against the
binomial shot floor the residual gives $\chi^2/\mathrm{dof}=1.37$ for $\dot q$
and $1.55$ for $\dot p$, leaving a device contribution of roughly $0.048$ and
$0.054$ on top of sampling.  The pull distributions agree: $52.5\%$ and $35\%$
of points fall within $1\sigma$ of the exact values, against the $68\%$ expected
if shot noise were the only error.  The underlying $\langle ZZ\rangle$
measurements show the same contraction toward zero seen in the energy sweep,
with a fitted slope of $0.983$ against exact---a $1.7\%$ systematic shrinkage
of the correlator that propagates through the shift rule into the gradients.
This is the same depolarising signature identified in
Section~\ref{sec:hw_budget}, and it is why the shift-pair construction
matters: both members of each pair ride one transpiled circuit with a shared
gate sequence and layout, so the coherent part of that error partially cancels
in the difference.

\subsection{Mitigation Ladder}\label{sec:hw_mitigation}

Error mitigation is often assumed to help.  To measure its effect on this
observable, the identical set of bound circuits was re-run at three settings:
raw execution; readout mitigation with dynamical decoupling and Pauli twirling;
and zero-noise extrapolation on top of the second setting.  The two additional
settings were run on a 12-point subset of the grid, chosen at even quantiles of
the exact energy so that the well, the flat region and the zero crossing are
all represented.  On that subset the first setting reproduces its own
full-grid error to within $0.0002$ in MAE, so the subset is representative of
the full sweep.

\begin{table}[htbp]
  \centering
  \caption{%
    Mitigation ladder on the common 12-point subset.  All three settings use
    the same bound circuits and 4096 shots per point, so only the mitigation
    strategy differs.  DD denotes dynamical decoupling; ZNE denotes zero-noise
    extrapolation.  All values in this subsection are read from the archived records of the
    three runs.
  }
  \label{tab:hw_ladder}
  \begin{ruledtabular}
  \begin{tabular}{lccccc}
    Setting & Resilience & DD, twirling & MAE & RMSE & Bias \\
    \colrule
    Raw                & 0 & off & $0.0794$ & $0.1054$ & $+0.0553$ \\
    Readout mitigation & 1 & on  & $0.0392$ & $0.0459$ & $+0.0285$ \\
    ${}+{}$ZNE         & 2 & on  & $0.0277$ & $0.0341$ & $+0.0043$ \\
  \end{tabular}
  \end{ruledtabular}
\end{table}

The ladder is not flat.  Mean absolute error falls by $51\%$ from raw execution
to readout mitigation, and by a further $30\%$ with zero-noise extrapolation,
for a total reduction of $65\%$ (Table~\ref{tab:hw_ladder},
Figure~\ref{fig:hw_ladder}).  The effect is strongest on the systematic
component: the bias falls from $+0.0553$ to $+0.0043$, a factor of $13$, while
the scatter about the exact values improves far less.  Mitigation is therefore
correcting the depolarising contraction identified in
Section~\ref{sec:hw_energy}, which is the error component that would otherwise
propagate into the learned gradients.

\begin{figure*}[htbp]
  \centering
  \includegraphics[width=\textwidth]{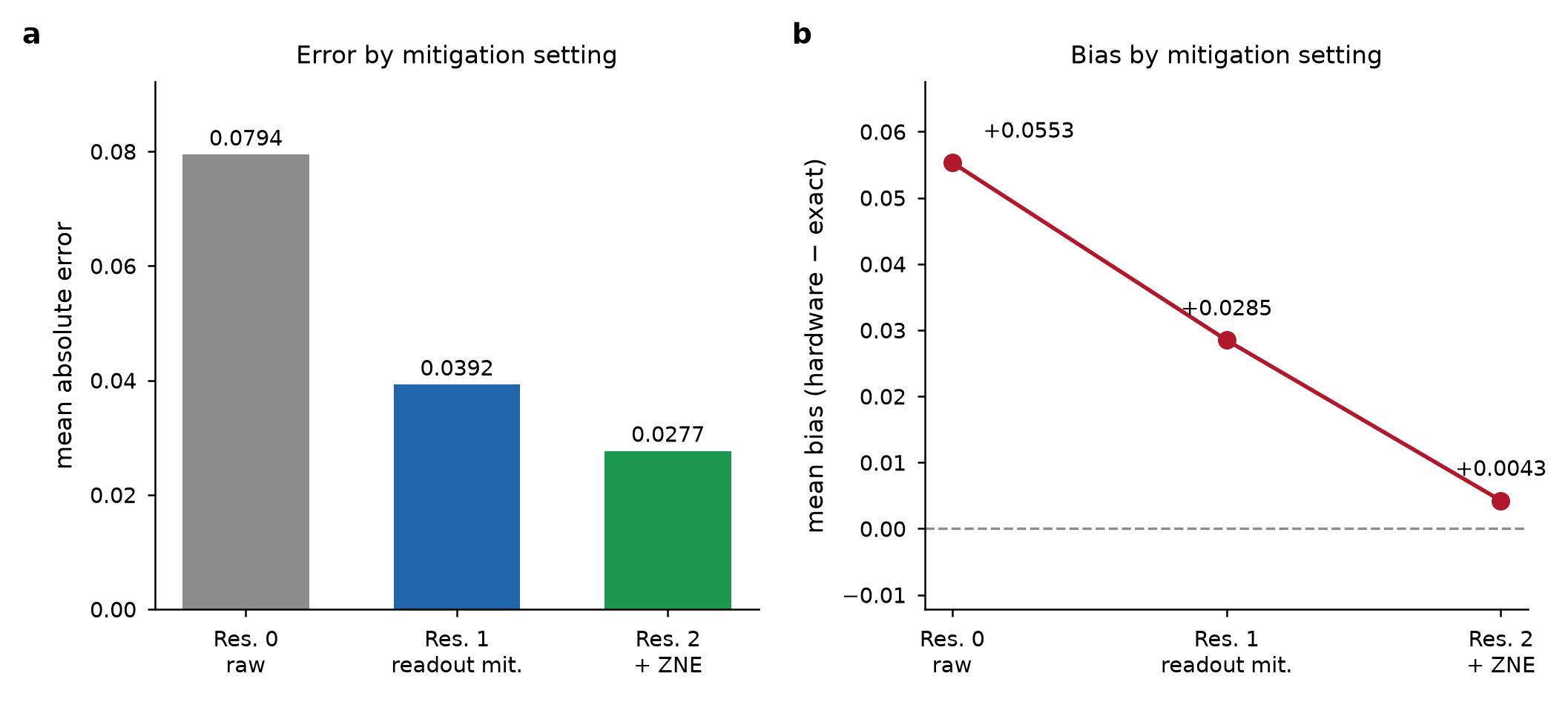}
  \caption{%
    \textbf{Measured effect of error mitigation on the Q-HNN energy.}
    (a)~Mean absolute error at each of the three mitigation settings.
    (b)~The corresponding mean bias, showing that mitigation acts on the
    systematic offset rather than on the scatter.  Each rung uses identical
    circuits and shot counts, so the differences are attributable to the
    mitigation strategy alone.
  }
  \label{fig:hw_ladder}
\end{figure*}

Two negative observations belong with this result.  First, the median relative
error does not improve monotonically across the ladder
($0.204 \rightarrow 0.123 \rightarrow 0.184$): zero-noise extrapolation
amplifies variance, and that penalty is largest where $|H|$ is smallest, which
is precisely where a relative measure is most sensitive.  Second, the
improvement is bought with device time.  The zero-noise rung cost $52$~s of
QPU time for 12 points against $15$~s for the raw rung, because extrapolation
requires executing the circuit at several noise amplifications.  At this
circuit size the ladder costs roughly three times the device time.  Both the benefit and the cost are properties of this circuit
size: Section~\ref{sec:hw_mitceiling} shows the same ladder degrading results at
depth $636$, at a device-time cost that rises faster still.

\subsection{Error Budget}\label{sec:hw_budget}

The calibration snapshot recorded with the sweep locates the dominant error
term.  For the two qubits used, the median single-qubit readout error is
$0.0164$, giving a two-qubit readout error of approximately $0.0324$, while
the \texttt{cz} gate error on the 135--139 pair is $1.01\times10^{-3}$, giving
approximately $0.0020$ across the two entangling gates.  Readout therefore
dominates gate error by a factor of about $16$ at this depth.

This ordering explains the shape of the ladder.  The large first step comes
from readout mitigation because readout is the dominant term; zero-noise
extrapolation, which targets gate noise, produces the smaller second step
because there are only two entangling gates over which to extrapolate.  The
Q-HNN is shallow enough that its error budget is a measurement problem rather
than a coherence problem: with $T_1$ of $245$~\textmu s and $T_2$ of
$111$~\textmu s against a depth-15 circuit, decoherence is not the binding
constraint.

The residual after full mitigation is at the sampling floor.  The mean absolute
error of the best setting, $0.0277$, is comparable to the median per-point
shot-noise standard deviation of $0.0245$ at 4096 shots.  No further mitigation
can improve this observable at this depth; only additional shots can.  This is
a favourable position for the architecture: the circuit is shallow enough that
its hardware error is set by how long one is willing to measure, not by the
fidelity of the device.

\subsection{Scope of These Measurements}\label{sec:hw_scope}

Three limits should be read with the numbers above.  The two additional ladder
settings rest on 12 grid points rather than 36, a consequence of the device-time
allocation; the agreement between subset and full grid for the shared setting
supports the subset but does not replace a full measurement.  Each setting was
executed once, so run-to-run calibration drift is not quantified and the step
sizes in Table~\ref{tab:hw_ladder} carry an unmeasured systematic.  All three
runs were submitted within a few minutes of one another against a single
calibration snapshot, and a ladder spanning a recalibration may differ.

What the preceding subsections establish is narrower than the simulation
results and prerequisite to them: the trained Q-HNN energy transfers to
hardware with a $3\%$-of-range error whose dominant component is a correctable
systematic bias.

\subsection{The Dissipative Channel on Hardware}\label{sec:hw_minl}

The MINL channel is the one component of the framework whose behaviour a
simulator cannot settle, because on a simulator the Born-rule measurement that
drives it is itself simulated.  On \texttt{ibm\_marrakesh} the channel executes
as a native dynamic circuit---mid-circuit measurement of the bath ancilla with
a conditional feedforward rotation, and ancilla reuse holding the register at
$N+1$ qubits rather than $2N$.  We ran the $N=3$ ring for
$0\ldots4$ Trotter steps and tracked the phase-space energy
$E_{\mathrm{ps}}=\sum_i\expval{X_i}^2$, against the conservative $\gamma=0$
circuit as a control.

Measured against that control, the dissipative arm suppresses
$E_{\mathrm{ps}}$ by $96.6\%$ between the first and fourth step, where the
noiseless simulation of the same circuits predicts $39.9\%$.  Read alone, the
excess invites the conclusion that measurement-induced dissipation is acting
on hardware.  It does not support that conclusion.  The dissipative circuit is
roughly twice the transpiled depth of its conservative control
($361$ against $172$ ISA layers at the fourth step), so a larger decay is
expected from decoherence irrespective of the channel.

The discriminating experiment is a control with the measurement machinery
intact and the dissipation removed: identical mid-circuit measurement, reset
and feedforward structure, with the conditional rotation angle set to zero, so
the circuit collapses the ancilla every step but applies an identity kick.
This null control is depth-matched to the dissipative circuit to within
$15$ ISA layers.  It decays at least as much as the dissipative
circuit at every step, retaining $0.009$ of its ideal $E_{\mathrm{ps}}$ at the
fourth step against $0.034$ for the dissipative arm
(Figure~\ref{fig:hw_minl}).  The loss is therefore attributable to the
mid-circuit measurement and the depth it costs, not to the dissipative
feedforward.  At the depths reached here, on this device, physical
measurement-induced dissipation is not separable from measurement back-action.

\begin{figure*}[htbp]
  \centering
  \includegraphics[width=\textwidth]{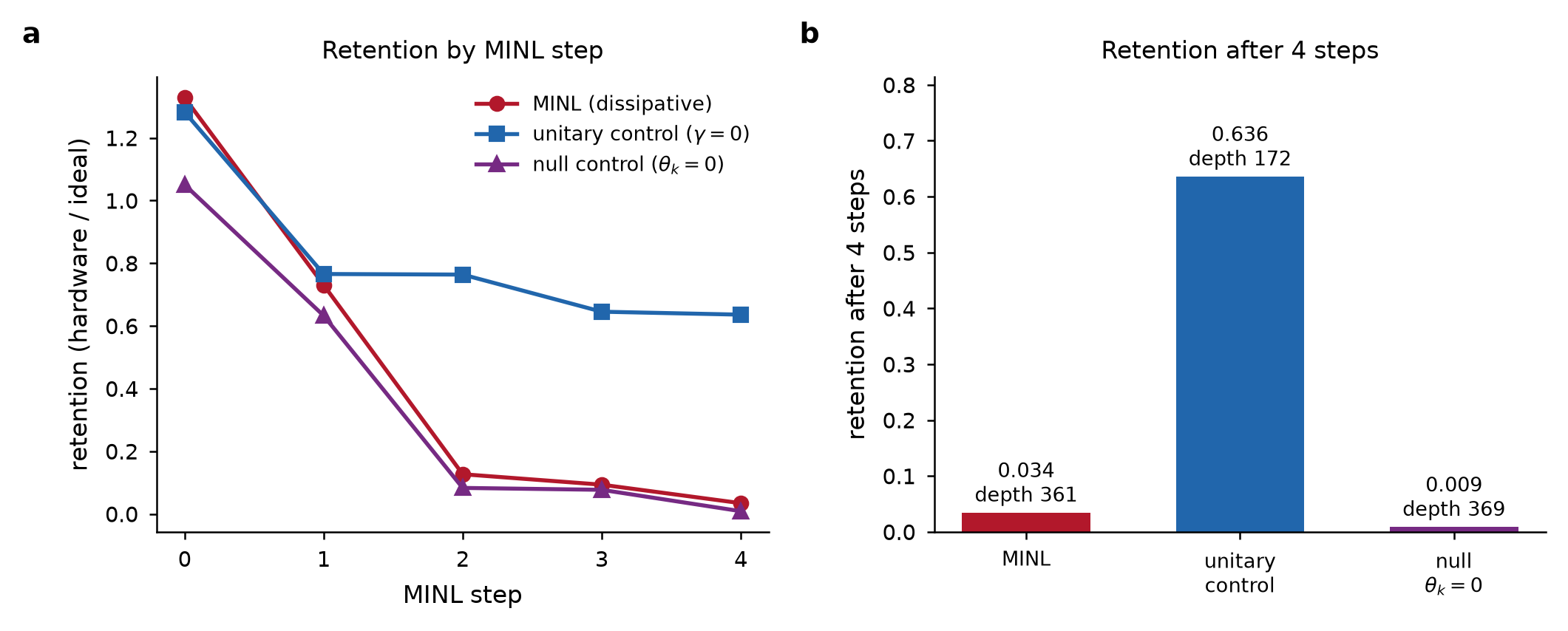}
  \caption{%
    \textbf{The dissipative channel is not separable from measurement
    back-action at this depth.}
    (a)~Phase-space energy retained on hardware relative to the noiseless
    simulation of the same circuit, $N=3$ ring.  The dissipative circuit
    (MINL) and the null control ($\theta_k=0$: same measurement structure,
    identity kick) decay together, while the shallower conservative control
    retains far more.  (b)~Transpiled depth for each arm.  The dissipative
    circuit is about twice the conservative control's depth, which is the
    confound the null control is built to remove.
  }
  \label{fig:hw_minl}
\end{figure*}

This is a limit of the present measurement, not a refutation of the channel.
Three changes would let the two contributions come apart: a shallower per-step
encoding, so the dissipative arm is not twice its control's depth; more Trotter
steps at lower per-step cost, separating rate from depth; and error mitigation
on the dynamic-circuit path, which the resilience settings used in
Section~\ref{sec:hw_mitigation} do not cover for mid-circuit measurement.
Until then, the simulation results of Section~\ref{sec:experiments} remain the
evidence for MINL dissipation, and the hardware establishes that the channel
runs natively rather than that its dissipation has been observed.

The transpilation cost sets where such an experiment can go.  Compiling the
conservative network circuit to the device at $N=4\ldots156$ shows the
practical ceiling to be depth, not qubit count: at $N=156$ routing inflates a
ring to $768$ two-qubit gates and the circuit-success proxy falls to
$0.169$, against $0.630$ at $N=100$ and $0.862$ at $N=32$.  The star
topology, whose hub has degree $N-1$ against a degree-3 heavy-hexagonal
lattice, is the most expensive family at the largest size tested
($0.810$ at $N=32$, against $0.862$ for the ring and $0.866$
for the chain), and its two-qubit count grows fastest.

This proxy counts two-qubit gate error alone, and Section~\ref{sec:hw_mirror}
shows it to be optimistic by a wide margin.

\subsection{Mirror Benchmarking at Network Scale}\label{sec:hw_mirror}

The network results of Section~\ref{sec:experiments} stop at $N=9$ because
statevector simulation is memory-bound near $N\approx24$, not because the
circuits stop at that size.  Hardware removes that bound but introduces
another: above $N\approx24$ there is no exact reference to score against.  We
therefore benchmark with mirror circuits---$U$ followed by $U^{\dagger}$---for
which the ideal output is $\ket{0\ldots0}$ at any $N$, so the return
probability is the circuit fidelity and no classical reference is needed.

We ran ring and chain topologies at $N=4,8,16,32,64,100$ and the star at
$N=4\ldots32$, two Trotter steps, $4096$ shots, as sixteen circuits in a single
job; the entire set was then repeated to measure run-to-run drift.  Both waves
cost $22$~s of device time each.

Fidelity falls far faster than the transpilation proxy predicts
(Figure~\ref{fig:hw_mirror}).  At $N=32$ the proxy gives $0.55$ for ring and
chain, while the measurement gives $0.0005$ and $0.0010$; at $N=64$ and
$N=100$ the return probability is not distinguishable from the shot-noise
floor.  Fitting $F=e^{-\alpha n_{2q}}$ over the detected points of each
topology yields effective errors per two-qubit gate of $0.0113$ (ring),
$0.0176$ (chain) and $0.0064$ (star), that is three to eight times the
calibrated \texttt{cz} error of $0.0023$.  The gap is the error the proxy
omits: readout on $N$ qubits, idle decoherence over a circuit of depth up to
$4019$, and crosstalk.  Each fit rests on the three detected sizes for that
topology, so the spread between topologies should be read as the range of a
noisy estimate rather than as a resolved ordering.

The practical consequence corrects the estimate above.  Measured mirror
fidelity exceeds $0.5$ only up to $N=8$ for all three topologies---not the
$N=64$--$100$ the gate-count proxy suggests.  A mirrored circuit is twice the
depth of the forward circuit it benchmarks, so the forward-only ceiling is
somewhat higher, but the correction is a factor of order two, not of order
ten.  Structure-preserving network circuits of the kind studied here are, on
present hardware, single-digit-$N$ objects.

Repeating the full set gave a mean absolute fidelity difference of $0.0120$
across the points above the floor, a mean relative drift of $4.0\%$ and a
maximum of $10.0\%$.  This is the first run-to-run figure in this work, and it
bounds the reproducibility of every single-shot hardware number reported
above; the two waves were submitted about a minute apart and share a
calibration epoch, so this is a lower bound on drift over longer intervals.

\begin{figure*}[htbp]
  \centering
  \includegraphics[width=\textwidth]{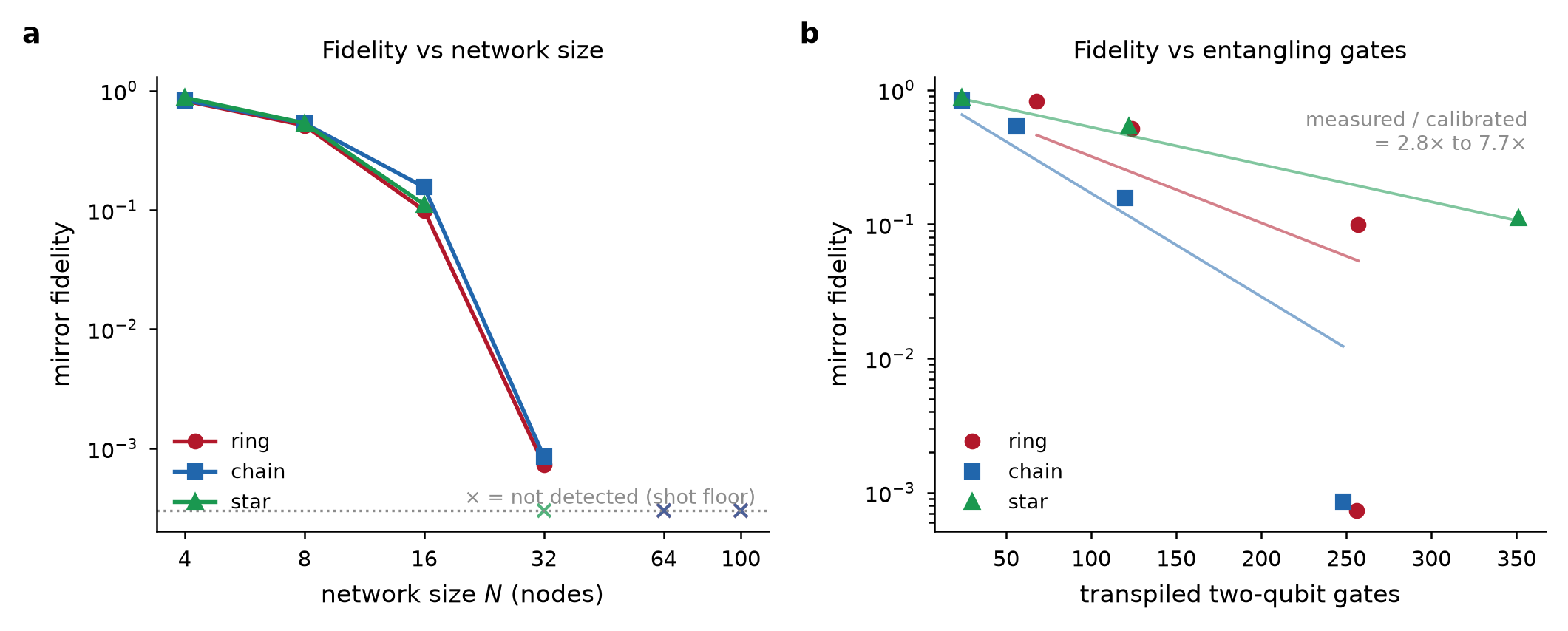}
  \caption{%
    \textbf{Usable network size is set by accumulated error, not by qubit
    count.}
    (a)~Mirror fidelity against network size for three topologies (solid,
    mean of two runs; error bars are half the run-to-run difference) with the
    two-qubit-gate proxy for the same circuits (dotted).  Open triangles mark
    points not distinguishable from the shot-noise floor.
    (b)~The same fidelities against transpiled two-qubit gate count.  The
    dashed line is a fit pooled across topologies ($0.0070$ per gate); the
    per-topology fits span $0.0064$--$0.0176$, three to eight times the
    calibrated \texttt{cz} error.
  }
  \label{fig:hw_mirror}
\end{figure*}

A mirror circuit measures whether a circuit survives the device; it is not a
physics result.  It shows neither energy conservation nor dissipation at any
$N$, and the physics claims of this work remain anchored to the sizes where an
exact reference exists.  What it establishes is the size at which such a
physics measurement could be attempted at all.  All $156$ qubits are
addressable, and every circuit up to $N=100$ transpiled and executed; the
limit is that the output stops carrying recoverable signal.

\subsection{Network Energy Conservation on Hardware}\label{sec:hw_conservation}

The mirror benchmark bounds where a physics measurement is possible; this
subsection performs one.  We evaluate the conservative ($\gamma=0$) network
energy $H$ at three points along a single trajectory, for ring, chain and star
topologies at $N=4$, $8$ and $16$---sizes at which an exact statevector
reference is still computable, so hardware can be scored rather than merely
executed.  Parameters were bound through the model's own binding routine, so
the hardware estimator and the reference share one code path.  All $27$
observable--circuit pairs ran as a single job at resilience level~1 with
dynamical decoupling and gate twirling, for $43$~s of device time.

Agreement is strong at $N=4$ and degrades sharply with size
(Figure~\ref{fig:hw_conservation}).  Mean relative error against the exact
reference rises from $4.6\%$ at $N=4$ to $17.5\%$ at $N=8$ and $36.2\%$ at
$N=16$, while the correlation with the reference falls from $r=0.989$ through
$r=0.925$ to $r=-0.279$.  The last figure is the important one: at $N=16$ the
measured energies are not a noisy version of the true energies, they are
uncorrelated with them.

Conservation along the trajectory tells the same story in the quantity the
theory constrains.  The discretised trajectory is not an exact flow, so the
reference energy itself drifts by $1.0$--$3.1\%$; hardware drift should be read
against that figure rather than against zero.  At $N=4$ measured drift is
$0.28$--$2.90\%$, comparable to the reference.  At $N=8$ it begins to exceed
it ($2.2$--$7.3\%$), and at $N=16$ it diverges, reaching $46.3\%$ for the star
topology against an exact drift of $0.97\%$.

\begin{figure*}[htbp]
  \centering
  \includegraphics[width=\textwidth]{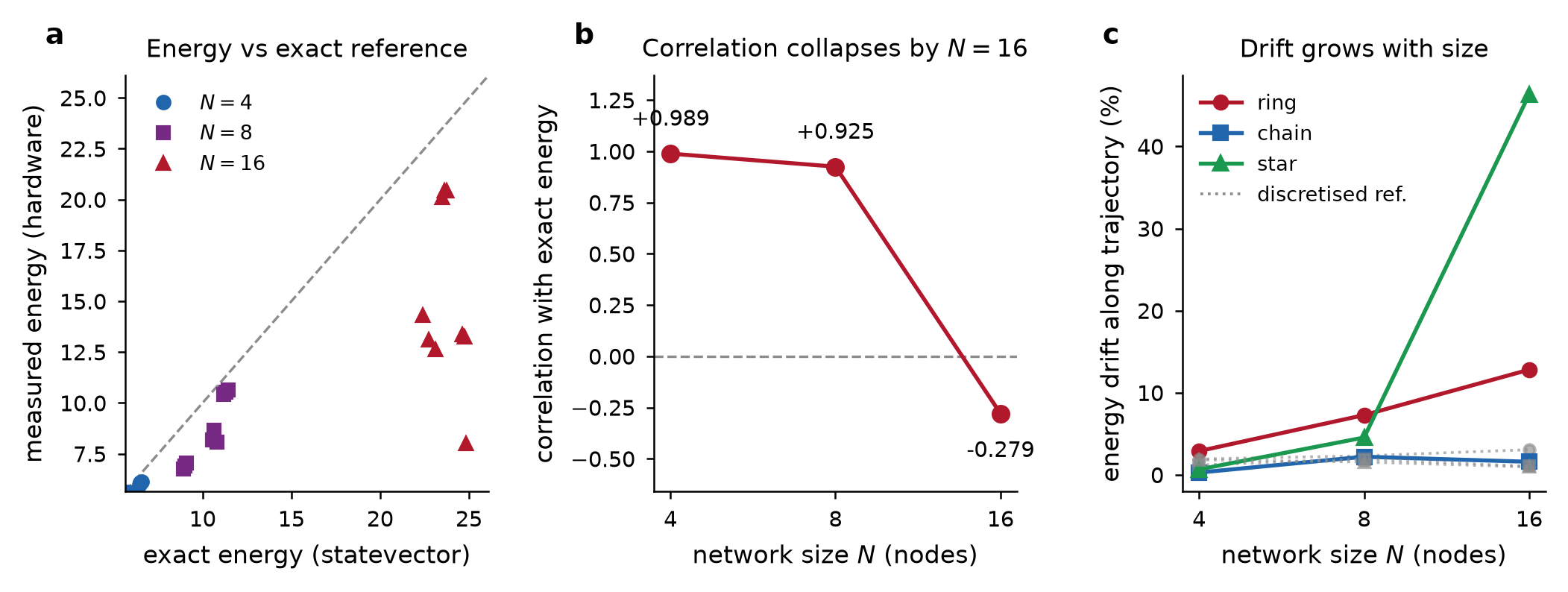}
  \caption{%
    \textbf{Network energy conservation transfers to hardware at $N=4$ and
    fails by $N=16$.}
    (a)~Hardware against exact energy for all $27$ points; error bars are
    estimator shot standard deviations and the dashed line is equality.
    (b)~Mean relative error by size, annotated with the Pearson correlation
    against the reference.
    (c)~Energy drift along each trajectory by topology, against the drift of
    the discretised reference trajectory itself (grey).
  }
  \label{fig:hw_conservation}
\end{figure*}

The mechanism is visible in the observable.  The network energy is a sum of
$O(N)$ Pauli terms, so its variance grows with $N$ while the signal does not,
and each term carries the readout error that Section~\ref{sec:hw_budget}
identified as the dominant contribution.  This measurement and the mirror
benchmark are independent---different observables, different circuits, no
shared reference---and they place the ceiling in the same region: mirror
fidelity at $N=16$ was $0.10$--$0.16$, and the energy estimate at $N=16$ is
uncorrelated with truth.  The simulated network results of
Section~\ref{sec:experiments}, which extend to $N=9$, sit inside the range
where hardware still tracks the reference at $N=8$ and outside it by $N=16$.

These measurements were taken at resilience level~1.  Whether a stronger
mitigation stack recovers the $N=16$ result is a question this data cannot
answer, and Section~\ref{sec:hw_mitceiling} answers it directly by re-measuring
the same states: it does not.

Two limits qualify this.  Each $(\text{topology},N)$ pair contributes one
trajectory of three points, so these are not statistics over initial
conditions; and the job was executed once, with the $4.0\%$ mean relative drift
of Section~\ref{sec:hw_mirror} the only available bound on reproducibility.

\subsection{The Depth Ceiling of Error Mitigation}\label{sec:hw_mitceiling}

Section~\ref{sec:hw_mitigation} measured a $65\%$ error reduction from the
mitigation ladder on the two-qubit, depth-$15$ Q-HNN circuit.  It is natural to
ask whether the same ladder rescues the network measurements of
Section~\ref{sec:hw_conservation}, whose errors at $N=8$ and $N=16$ sit $72$ and
$353$ times above the shot-noise floor and are therefore bias-dominated---and
bias is what zero-noise extrapolation targets.

We tested this directly.  The $N=8$ and $N=16$ conservation points were
re-measured on identical states and parameters, changing only the mitigation
stack: resilience level~2 with ZNE using the probabilistic-error-amplification
(PEA) amplifier at noise factors $\{1,3,5\}$, XY4 dynamical decoupling, and gate
twirling over $32$ randomisations.  PEA learns the twirled Pauli--Lindblad noise
model of each entangling layer and amplifies by injecting noise drawn from that
model, rather than by synthetic gate folding.  The states were regenerated from
the original seeded construction and verified against the stored exact energies
to $0.0$ before submission, so mitigation is the only variable.

\emph{The stronger stack made both sizes worse.}  Mean absolute error rose by
$11.2\%$ at $N=8$ ($1.7556\rightarrow1.9522$) and by $35.8\%$ at $N=16$
($8.6383\rightarrow11.7325$).  Correlation with the exact energies fell from
$0.9251$ to $0.6803$ at $N=8$ and remained negative at $N=16$
($-0.2793\rightarrow-0.2826$).  The contraction of the measured energies toward
zero deepened rather than lifting---at $N=16$ from $0.636$ to $0.505$ of the
exact magnitude---and the reported uncertainties grew by a factor of $2.1$--$2.5$
(Figure~\ref{fig:hw_mitceiling}a,b).

The mechanism is visible in the amplified circuits themselves.  ZNE extrapolates
to zero noise from executions at $1\times$, $3\times$ and $5\times$ amplification,
which presumes those executions still carry signal.  At the decay rates measured
independently by mirror benchmarking (Section~\ref{sec:hw_mirror}), the
$N=16$ circuits retain between $0.05\%$ and $6.5\%$ of their signal at
$3\times$ and less at $5\times$, depending on topology
(Figure~\ref{fig:hw_mitceiling}c).  The fit is then anchored on two
uninformative points and the extrapolation is unconstrained; the inflated
uncertainties are that instability made visible.  Mitigation cannot recover a
signal that amplification has already destroyed.

\begin{figure*}[htbp]
  \centering
  \includegraphics[width=\textwidth]{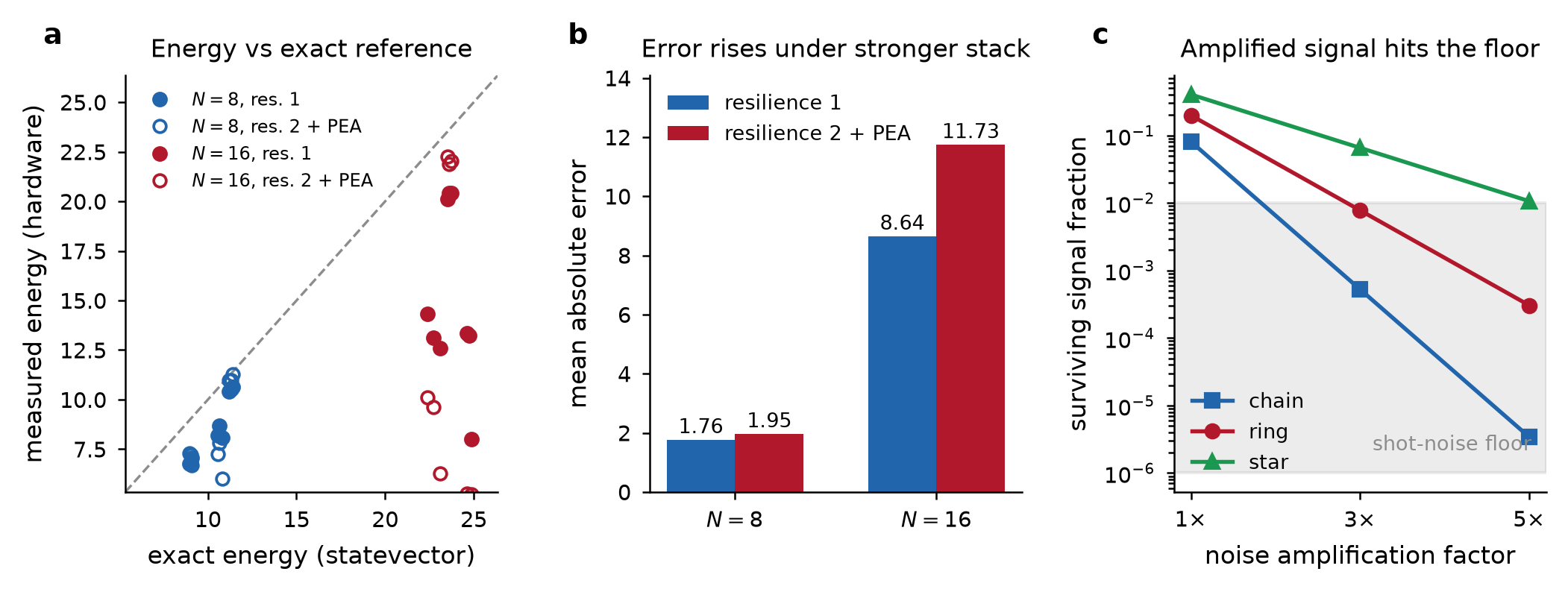}
  \caption{%
    \textbf{Error mitigation has a depth ceiling.}
    (a)~Measured against exact energy for the same states under resilience~1
    (filled) and resilience~2 with PEA and XY4 (open); dashed line is equality.
    (b)~Mean absolute error rises at both sizes under the stronger stack.
    (c)~Surviving signal fraction at each ZNE amplification factor for $N=16$,
    computed at the mirror-measured decay rates, against the shot-noise floor.
    All values in this subsection are read from the archived record of this
    run.
  }
  \label{fig:hw_mitceiling}
\end{figure*}

Two conclusions follow.  The $N=16$ failure of
Section~\ref{sec:hw_conservation} is not an artifact of insufficient error
mitigation: the standard ladder, applied at its strongest available setting,
does not move it.  This strengthens rather than weakens the ceiling claim, since
the limit is accumulated physical error rather than a correctable systematic.
And the ladder's own scope is bounded: that improvement was measured at
depth $15$ and does not extend to depth $636$.  We did not re-measure $N=4$, so the crossover depth at which
mitigation stops helping is bracketed between these two values rather than
located.

One caveat of practice is worth recording, because it governs what such a test
costs.  The PEA amplifier performs a noise-learning pass per entangling layer
before it amplifies, and this dominates the charge: the run consumed $330$~s of
device time against a $90$~s estimate extrapolated from the plain-ZNE overhead
factor.  Budgeting PEA as a multiple of gate-folding ZNE understates it
substantially at these depths.

\subsection{Coupled Nonlinear Pendulums, $N=2$ to $12$}\label{sec:hw_pendulum}

The preceding subsections evaluate \emph{trained} circuits, where the depth of
the variational ansatz sets the size ceiling.  A second route to the same
physics avoids that ceiling entirely.  Because the isomorphic Hamiltonian
mapping is structural, the coefficients of the coupled-pendulum Hamiltonian
follow from $K$ and $g$ directly, with no optimisation: the energy can be
\emph{measured} rather than learned.  Combining the on-site correspondence
$\expval{Z_i}=\cos\phi_i$ with Equation~\eqref{eq:zzxx} gives
\begin{equation}
  H = \tfrac12\textstyle\sum_i\omega_i^2
      + g\bigl(N-\sum_i\expval{Z_i}\bigr)
      + K\bigl(|E| - \sum_{(i,j)\in E}[\expval{Z_iZ_j}+\expval{X_iX_j}]\bigr),
  \label{eq:H_measured}
\end{equation}
exact at every $N$ (verified to $8\times10^{-15}$).  The circuits are a single
layer of $R_y$ rotations---transpiled depth $4$---so the accumulated-error
ceiling of Section~\ref{sec:hw_mirror} does not bind.

We evaluated chains of $N=2,4,8$ and $12$ pendulums at five global phase shifts
each, reading the three commuting observable groups of
Equation~\eqref{eq:H_measured} for $60$ observable--circuit pairs in one job at
resilience level~1, for $80$~s of device time.  Agreement with the exact
statevector values is $r\geq0.9997$ at every size, and---unlike the trained
network of Section~\ref{sec:hw_conservation}---\emph{it does not degrade with
$N$}: the mean relative error in the reconstructed $H$ is $2.53\%$ at $N=2$,
$1.06\%$ at $N=4$, $1.39\%$ at $N=8$ and $0.43\%$ at $N=12$
(Figure~\ref{fig:hw_pendulum}a,c).  The apparent improvement with size is a
scale effect: $H$ grows with $N$ while the per-qubit readout error does not, so
the same absolute error ($0.02$--$0.05$) is a smaller fraction of a larger
number.  The relevant statement is that accuracy is flat in $N$ where the
trained-circuit measurements fell apart.

This experiment also settles on hardware the read-out question raised in
Section~\ref{sec:net_experiments}.  Under a global phase shift the classical coupling
energy is exactly constant.  Measured with the corrected $ZZ+XX$ observable it
is constant to $0.05$--$0.17$ across all four sizes, consistent with shot noise.
Measured with the $ZZ$-only read-out used throughout the simulation results, it
varies by $0.76$ at $N=2$, rising monotonically to $9.25$ at $N=12$
(Figure~\ref{fig:hw_pendulum}b): \emph{the symmetry violation grows with system
size}.  This is the one place in this work where hardware measurement identifies
a correctable defect in the framework rather than a limitation of the device.

\begin{figure*}[htbp]
  \centering
  \includegraphics[width=\textwidth]{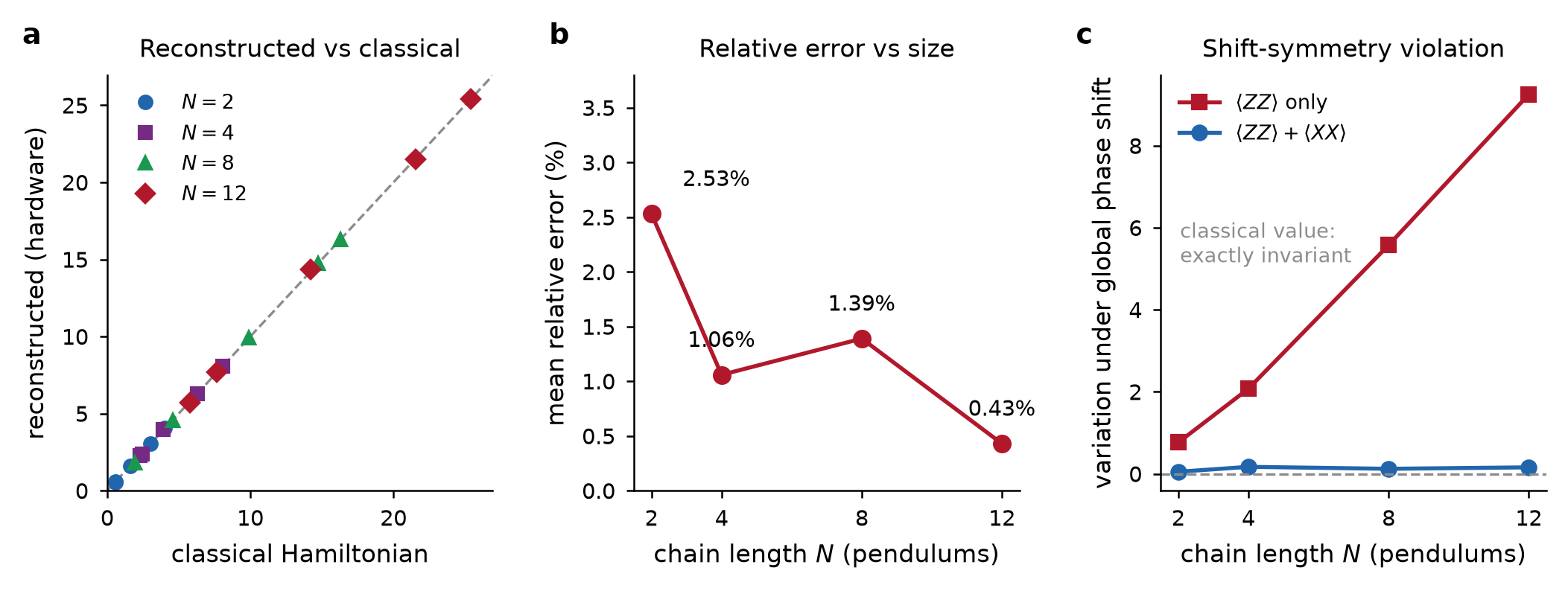}
  \caption{%
    \textbf{Coupled nonlinear pendulums measured directly on hardware,
    $N=2$--$12$.}
    (a)~Hardware against exact energy for all $20$ states; dashed line is
    equality.
    (b)~Variation of the coupling energy under a global phase shift, which the
    classical Hamiltonian leaves exactly invariant.  The corrected $ZZ+XX$
    read-out is constant to within shot noise; the $ZZ$-only read-out's
    violation grows with $N$.
    (c)~Mean relative error in the reconstructed $H$, annotated with the
    correlation against the exact values.  Accuracy is flat in $N$ because the
    circuits are depth-$4$ regardless of size.
    All points: \texttt{ibm\_marrakesh}, one job, $80$~s of device time, no
    training.
  }
  \label{fig:hw_pendulum}
\end{figure*}

Two limits apply.  No training is involved, so this tests the mapping and the
device, not the learning procedure---it is complementary to
Section~\ref{sec:hw_conservation}, not a replacement for it.  And the kinetic
term of Equation~\eqref{eq:H_measured} is evaluated classically from the state;
only the configurational terms are measured.

\section{Discussion}\label{sec:discussion}

\subsection{The IHM as a Constructive Design Principle}

The Isomorphic Hamiltonian Mapping provides structural guarantees not through
penalty terms or post-hoc projection, but through quantum physics.
Unitarity of $U_J$ enforces $\mathbf{J}=-\mathbf{J}^\top$; positive probability
of Born-rule measurement enforces $\mathbf{R}\succeq 0$.
The Q-pHNN therefore produces a learning architecture in which the
port-Hamiltonian constitutive relations are \emph{impossible to violate},
regardless of the parameter values optimised by BFGS or COBYLA.
This contrasts with classical pHNNs~\cite{desai2021}, which must project
$\mathbf{J}$ and $\mathbf{R}$ into their respective matrix cones at every
gradient step---projections that add computational overhead, can fail for
non-square Jacobians, and require careful numerical conditioning.
The recent SympNets~\cite{jin2020} achieve symplecticity through a product
of elementary symplectic layers; the Q-HNN achieves it through the algebraic
isomorphism of the IHM, requiring no special architectural design beyond the
choice of observable.

\subsection{Connection to Open Quantum Systems}

The MINL measurement step has a precise relationship to open quantum system
theory~\cite{lindblad1976}.  A projective measurement on the ancilla
followed by feedforward constitutes a Kraus map
$\mathcal{E}(\rho)=\sum_b K_b\rho K_b^\dagger$ with
$K_b=P_b\otimes R_x(\theta_k)^b$, which is a completely-positive
trace-preserving (CPTP) map---the elementary building block of the
Lindblad master equation.  The Q-pHNN therefore learns a discrete-time
approximation of Lindblad dynamics, with the jump rate controlled by
$\theta_R$ and the jump direction by $\theta_k$.
This connection suggests a natural extension to genuine open quantum system
identification, learning the dissipator structure of a qubit coupled to a
thermal bath from trajectory observations~\cite{popovych2022}.

\subsection{Positioning Relative to Prior Work}

Quantum Hamiltonian learning~\cite{mitarai2018} and variational quantum
eigensolvers~\cite{cerezo2021} address the inverse problem of identifying the
Hamiltonian of a \emph{quantum} system from quantum measurements.
The Q-pHNN framework is distinct: it operates in \emph{classical} phase space,
learning the dynamics of \emph{classical} physical systems from classical
trajectory data.  The quantum circuit is a parameterised function approximator,
not a model of the measurement process.
Physics-informed neural networks~\cite{raissi2019} and their quantum
analogues penalise violation of differential equations in a loss function;
the Q-pHNN encodes the geometric structure of the equation directly into the
circuit architecture, providing hard constraints rather than soft penalties.
Compared to Neural ODEs~\cite{chen2018}---which provide continuous-time
dynamics but no structural guarantees---the Q-pHNN is specialised to
Hamiltonian and port-Hamiltonian systems and provides exact energy conservation
or monotone dissipation by circuit construction.
The symplectic parameter-shift rule builds on the standard PSR~\cite{schuld2019},
extending it to data-encoding gates to compute exact Hamilton's equations---a
computation that data-driven methods such as SINDy~\cite{brunton2016} and
universal differential equations~\cite{rackauckas2020} approximate numerically.

\subsection{Why Quantum Circuits for Classical Dynamics?}

A natural question is whether a classical neural network would not suffice.
The Q-pHNN's advantages are structural and forward-looking, not merely
computational at present scale:
\begin{enumerate}[nosep]
  \item \textbf{Exact analytic gradients from data-encoding gates.}
    The Parameter-Shift Rule provides exact partial derivatives of the
    circuit observable with respect to phase-space coordinates from two
    circuit evaluations, without automatic differentiation or finite
    differences.  This is inherent to the quantum computing paradigm and
    enables exact BFGS convergence in 8--13 iterations.
  \item \textbf{Hardware-native Lindblad dynamics.}
    The MINL channel is not a software approximation of dissipation: it is
    a genuine discrete-time CPTP map, expressible on any device that
    supports mid-circuit measurement and classical feedforward
    (IBM Eagle/Heron, IonQ Aria, Quantinuum H2).
    Expressibility is not the same as measurability, however: on an IBM Heron
    processor the channel executes as a native dynamic circuit, but the energy
    decay it produces cannot be separated from the back-action of the
    mid-circuit measurement that implements it (Section~\ref{sec:hardware}).
    Demonstrating the dissipation as a physical effect, rather than as a
    correct simulation of one, requires a device or an encoding at which the
    two contributions come apart.
  \item \textbf{Long-term quantum advantage for large systems.}
    As quantum hardware matures, the same Q-pHNN circuit---trained here
    in classical statevector simulation---may be executed on real hardware,
    potentially exploiting genuine quantum parallelism for high-dimensional
    phase spaces where classical structure-preserving methods are
    computationally expensive.  Mirror benchmarking in
    Section~\ref{sec:hardware} places the present ceiling for the network
    circuit at single-digit $N$ on a $156$-qubit device: every circuit up to
    $N=100$ executes, but the output ceases to carry recoverable signal well
    before that.  The limit is accumulated error, not qubit count.
\end{enumerate}

\subsection{Limitations}

\paragraph{Circuit expressibility.}
The 2-qubit, 4-parameter circuit ($L=1$) is sufficient for the nonlinear
pendulum: the layer ablation (Section~\ref{sec:ablation}) shows that $L=2,3$
achieve identical accuracy, confirming that $L=1$ saturates the expressibility
for this 2-DOF system.  The $1.35\%$ energy drift (after scale correction)
is the circuit's irreducible approximation residual---the $ZZ$ observable
cannot perfectly represent the non-quadratic pendulum level sets with only 4
parameters.  For the damped oscillator, the residual trajectory phase offset
is an honest expressibility bound: the 4-parameter $ZZ$ ansatz cannot
simultaneously represent the large-amplitude initial transient and the slow
decaying tail.  For higher-dimensional or strongly nonlinear systems,
increasing $L$ or adopting richer observables will be necessary.

\paragraph{MINL stochasticity.}
Each Q-pHNN forward pass is stochastic, since dissipation is produced by
Born-rule measurement.  With finite shots the single-oscillator loss variance
is substantial and COBYLA converges more slowly than a gradient method; the
network study controls this by reading the phase-space energy from several
thousand shots per step ($4000$ for the $N=6$ demonstration,
$3000$ for the scaling sweep).  Generalised parameter-shift rules for quantum
channels~\cite{schuld2019} could, in principle, provide analytic gradients
even through the measurement step.  This stochasticity is not a defect but the
physical signature of measurement-induced dissipation: the same shot noise is
what a real device would exhibit.

\paragraph{Damping model.}
The MINL channel realises per-node amplitude damping with rate $\gamma_i$ set by
the ancilla rotation angle $\theta_{R,i}$.  Real systems may exhibit
state-dependent or anisotropic dissipation.  A natural extension is to make
$\theta_{R,i}$ a function of the measured state, at the cost of an additional
controlled rotation per Trotter step.

\paragraph{From one to many degrees of freedom.}
The single-oscillator models are lifted to $N$-node coupled-phasor networks by
the topology-entangled QGNN of Section~\ref{sec:net_architecture}, which uses
one qubit per node and places two-qubit entanglers only on physical coupling
edges.  Both IHM properties survive the
lift: unitarity of the graph-structured $U_J$ conserves the network energy
exactly ($|\dot H|=0$ in floating point at every $N$ and topology), and the
multi-ancilla MINL channel dissipates the network phase-space energy by
$92$--$98\%$ across ring, star, and chain networks at $N\in\{3,6,9\}$---the
measurement-based analogue of the passivity inequality, with one bath ancilla
per node.

\paragraph{Network scaling and barren plateaus.}
Statevector simulation of the QGNN is exact but scales as $2^{N}$; the present
study reaches $N=9$ nodes (plus $N$ bath ancillas for the MINL channel) in
simulation.  On hardware the binding constraint is different and was measured
directly: accumulated gate and readout error, not memory, caps the usable
network at single-digit $N$ (Section~\ref{sec:hw_mirror}).  A more
fundamental obstacle
for larger $N$ and deeper circuits is the \emph{barren plateau}
problem~\cite{mcclean2018}: gradient magnitudes of cost functions defined on
random parameterised quantum circuits decay exponentially with the number of
qubits.  Topology-aware initialisation strategies, layerwise training, or
hardware-efficient parameterisations~\cite{cerezo2021} that exploit the network
structure are natural mitigations.

\subsection{Future Directions and Cross-Domain Applications}
\label{sec:future}

The Q-pHNN framework establishes a general-purpose bridge between
port-Hamiltonian system theory and quantum machine learning.
Below we outline the most promising application domains and the specific
extensions required for each.

\paragraph{Multi-omic twins of the cell.}
A twin fitted to biochemical time-series would inherit a port-Hamiltonian
budget when the measured pools exchange material through a regulatory graph and
lose it through turnover~\cite{karlebach2008}: a free-energy-like scalar plays
the role of $\mathcal H$, exchange fluxes populate $\mathbf J$, and degradation
rates populate $\mathbf R$.  The topology-entangled QGNN maps onto such a graph
with one qubit per molecular pool and $ZZ$ entanglers on co-regulated pairs, so
the learned damping rates $\boldsymbol\gamma$ would become an estimate of
turnover half-lives read off the fitted twin rather than assumed.

A candidate falsifiable observable is the aggregate transcript-to-protein cascade lag,
which such a twin should reproduce from data alone.  For metabolic networks,
the resilience framework of~\cite{gao2016} predicts universal tipping behaviour
near a critical damping threshold, which a fitted $\gamma$ could test without
assuming the network equations \textit{a priori}.  The universal differential
equations framework~\cite{rackauckas2020} offers a natural interface, with UDE
residuals parameterised by Q-pHNN circuits.  Establishing any of this as an
application would require appropriate longitudinal data, matched classical
baselines, and validation at scales beyond exact statevector simulation.

\paragraph{Neural twins and brain-computer interfaces.}
The rotational structure reported in motor cortex population
recordings~\cite{shenoy2013} is well described by a Hamiltonian flow, with a
rotational kinetic energy as the conserved quantity and movement termination as
a dissipative transition; a twin fitted to those recordings would therefore
carry the same $(\mathbf J,\mathbf R)$ split.  A Q-HNN decoder built on such a
twin would have energy monotonicity as a structural certificate rather than a
trained behaviour, which would bound the trajectory drift that afflicts
recurrent decoders.  The LFADS architecture~\cite{pandarinath2018} shows that
sequential auto-encoders infer low-dimensional neural manifolds; a Q-pHNN
decoder on the same manifold would additionally carry port-Hamiltonian
geometry, a constraint LFADS does not impose.  A decoding claim would require
considerably more than the present work supplies: spectral controls the twin
must pass, closed-loop and real-time evaluation, and execution at population
scale, which is bounded by the fidelity ceiling measured in
Section~\ref{sec:hw_mirror} rather than by qubit count.

\paragraph{Quantitative finance and stochastic market dynamics.}
Financial markets exhibit conservation-like laws (put-call parity, absence of
arbitrage) and irreversible dissipation (volatility mean-reversion, liquidity
decay).  Quantum Monte Carlo algorithms provide quadratic speedups for
derivative pricing~\cite{rebentrost2018b}, and quantum optimisation has been
applied to portfolio selection~\cite{orus2019}.
The Q-pHNN framework adds a structural layer: the $\mathbf{J}$-channel encodes
no-arbitrage constraints as a conserved symplectic structure, while the
$\mathbf{R}$-channel learns the mean-reversion rate $\gamma$ from implied
volatility surfaces.
For stochastic volatility models (Heston, SABR), the state-dependent damping
extension---$\gamma(\xvec)$ parameterised by a second PQC---directly gives
a quantum analogue of the mean-reversion function, with the IHM guaranteeing
that the learned model never produces explosive volatility.

\paragraph{Cardiac electrophysiology and medical digital twins.}
Cardiac action potential propagation is a high-dimensional dissipative
port-Hamiltonian system~\cite{trayanova2011}: ionic channels are ports, the
membrane potential is the energy variable, and conduction velocity is
conserved across wave propagation fronts.  Re-entrant arrhythmias correspond
to limit cycles in the phase-space of the cardiac port-Hamiltonian
system---identifiable from the Q-pHNN's learned energy function as persistent
non-decaying orbits.
The topology-entangled QGNN maps the Purkinje fibre network: one qubit per
anatomical node (sinoatrial node, atrioventricular node, bundle branches),
$ZZ$ entanglers on conduction edges, and the energy-monotone dissipation
fraction as a real-time arrhythmia risk index.

\paragraph{Engineering control systems and robotics.}
The SE(3) Hamiltonian dynamics of rigid-body robots have been modelled by
neural ODEs~\cite{duong2021}; the Q-pHNN provides the same capability with
structural guarantees.  A robot arm modelled as a conservative Q-HNN (joint
torques as port inputs) guarantees that learned energy functions satisfy
Lyapunov stability conditions by construction---a property absent from
unconstrained neural ODE controllers.  The $\mathbf{R}$-channel models joint
friction; the learned $\gamma_i$ per joint provides a self-calibrating
digital twin that identifies wear without sensor instrumentation.

\paragraph{Methodological extensions.}
Three near-term developments would substantially extend the framework:
(i)~\textit{State-dependent dissipation:} parameterise $\gamma(q,p)$ as a
second PQC on the same qubits, evaluated with 4 additional PSR evaluations;
(ii)~\textit{Multi-frequency encoding:} use Fourier feature maps~\cite{schuld2021}
in the data-encoding gates to extend the expressible function class beyond
trigonometric polynomials;
(iii)~\textit{Noise-aware training:} replace noiseless statevector simulation with
a shot-based noisy simulator to train directly in the shot-noise regime,
enabling transfer to real hardware without hyperparameter re-tuning.

\section{Conclusion}\label{sec:conclusion}

We have introduced Quantum Port-Hamiltonian Neural Networks (Q-pHNNs), a
framework that exploits a formal algebraic isomorphism between
Port-Hamiltonian systems and quantum circuit primitives to learn conservative
and dissipative classical dynamics in a structure-preserving manner.

The Isomorphic Hamiltonian Mapping (Definition~\ref{def:ihm}) identifies
the skew-symmetric structure $\mathbf{J}$ with unitary gate evolution and
the positive-semidefinite structure $\mathbf{R}$ with Born-rule projective
measurement, providing structural guarantees by circuit construction
(Theorem~\ref{thm:iso}).
The symplectic parameter-shift rule (Theorem~\ref{thm:psr}) enables exact
analytic computation of Hamilton's equations from data-encoding gates,
requiring only four circuit evaluations per phase-space point.

Three experiments validate the framework.
Q-HNN achieves $\mathbf{1.35\%}$ relative energy drift on the nonlinear pendulum
with a 4-parameter, 2-qubit circuit using a St\"ormer--Verlet symplectic integrator
over 300 steps---a $9.1\times$ improvement over the first-order Euler
baseline (12.3\%).
A jointly-optimised energy scale $s^*=1.335$ corrects the gradient-magnitude
mismatch between the bounded $\langle ZZ\rangle\in[-1,1]$ observable and the
true pendulum energy, recovering the correct oscillation frequency and reaching
a validation $\dot{q}$ MSE of $0.040$ in only 8 BFGS iterations.
The layer ablation confirms $L=1$ saturates expressibility for this 2-DOF system:
additional layers add cost but no accuracy gain, supporting the IHM's claim that
structure, not depth, is the key design ingredient.

The Q-pHNN achieves $100\%$ energy monotonicity across 30 independent MINL
trajectory runs on the damped oscillator, confirming that Born-rule measurement
accurately models irreversible dissipation.
Lifted to networks, the multi-ancilla MINL channel dissipates the phase-space
energy by $92$--$98\%$---monotonically at every step---across ring, star, and
chain topologies at sizes $N\in\{3,6,9\}$ in simulation, the measurement-based
analogue of the passivity inequality, while the conservative mode conserves
energy to machine precision at every scale.

These results establish quantum circuits as structure-preserving surrogates
for classical port-Hamiltonian systems.  The framework is extensible to
multi-dimensional systems, state-dependent dissipation, and---via the
QGNN lift---arbitrarily large coupled-oscillator networks while preserving
all IHM structural guarantees.
The mid-circuit measurement and classical-feedforward operations used
throughout are compatible with current measurement-capable hardware
(superconducting and trapped-ion processors).  We take a first step along that
path here: on an IBM Heron processor the trained Q-HNN energy surface and its
parameter-shift gradients reproduce their simulated values, while the
dissipative channel---though it executes natively---does not yet yield a
dissipation signal separable from the measurement back-action that implements
it.  The remaining gap is quantitative and is set by circuit depth rather than
by qubit count.
In the longer term, the Q-pHNN framework provides a template for quantum
digital twins of neuromechanical, biochemical reaction-network, and
engineering control systems---domains where the interplay of conservative
mechanics and irreversible dissipation is fundamental.

The Q-HNN experiments are deterministic under fixed random seeds;
the Q-pHNN dynamic circuit is intrinsically stochastic, as each forward pass
draws Born-rule measurement outcomes, so its reported metrics vary slightly
between runs at fixed seed.  All experiments use only the architectures and
hyperparameters specified in Sec.~\ref{sec:methods} and
Appendix~\ref{app:hparams}.

\section*{Data and Code Availability}

All hardware runs reported in Section~\ref{sec:hardware} were executed on the
IBM Quantum Platform.  The complete record of each run---the IBM job
identifier, the device calibration snapshot taken at submission, the
transpiled circuit statistics, the raw expectation values, and the exact
statevector reference against which they are scored---is archived with the
source code, together with the software that regenerates every figure and an
automated consistency check that verifies each numeric literal in this paper
against its archived value.  Nothing quoted here was transcribed by hand. 
Most updated experiment codes are available at \\
\url{https://github.com/mindverse-computing/quantum-port-hamiltonian}

\bibliography{references}

\appendix

\section{Hyperparameters and Experimental Details}\label{app:hparams}

All experiments are fully specified by the architectures and settings of
Section~\ref{sec:methods}.  Table~\ref{tab:hparams} collects all
hyperparameters; no hyperparameter tuning was performed---all values
were fixed before running the experiments reported here.

\begin{table}[htbp]
  \centering
  \small
  \begin{tabular}{lll}
    \toprule
    \textbf{Parameter} & \textbf{Q-HNN} & \textbf{Q-pHNN (MINL)} \\
    \midrule
    Qubits              & 2 (sys)           & 2 (sys+anc)       \\
    Circuit layers $L$  & 1                 & ---               \\
    Trainable params    & 4                 & 3                 \\
    Training samples $N_s$ & 200            & 6 (trajectory)    \\
    Train/val split     & 80/20             & ---               \\
    Optimiser           & BFGS              & COBYLA            \\
    Max iterations      & 200               & 80                \\
    Shots per eval      & ---               & 30                \\
    Integrator          & St\"ormer--Verlet & Euler             \\
    Rollout steps $K$   & 300               & 6                 \\
    Time step $\Delta t$& 0.05              & 1.0               \\
    Initial condition   & $(0.8,\,0.0)$     & $(1.5,\,0.0)$     \\
    Random seed         & 42                & ---               \\
    \midrule
    \multicolumn{3}{l}{\textit{Ground-truth generation}} \\
    Integrator          & RK4               & RK4               \\
    System              & Nonlinear pendulum & Damped harmonic oscillator \\
    True $\gamma$       & ---               & 0.3               \\
    \midrule
    \multicolumn{3}{l}{\textit{Quantum evaluation}} \\
    Objective           & $\langle ZZ\rangle$ expectation & Born-rule ancilla measurement \\
    Simulation          & Exact statevector & Exact statevector \\
    \bottomrule
  \end{tabular}
  \caption{Hyperparameters for the single-oscillator experiments.
    The network QGNN conservative experiment uses $N=3$ nodes, $L=2$ layers,
    60 training samples, ring coupling with strength 1.0, and BFGS to
    convergence (45 iterations); the MINL network scaling study uses
    $N\in\{3,6,9\}$, ring/star/chain topologies, 8 Trotter steps, and 3000
    measurement shots per step, with no optimiser.
    All seeds fixed before running.}
  \label{tab:hparams}
\end{table}

\section{Notation Summary}\label{app:notation}

Table~\ref{tab:notation} collects the symbols used throughout.

\begin{table}[htbp]
  \centering
  \small
  \begin{tabular}{ll}
    \toprule
    \textbf{Symbol} & \textbf{Meaning} \\
    \midrule
    $\mathbf{J}$       & Skew-symmetric interconnection matrix of the pH system \\
    $\mathbf{R}$       & Positive-semidefinite dissipation matrix \\
    $H(\mathbf{x})$    & Scalar Hamiltonian (energy function) \\
    $\mathbf{G}(\mathbf{x})u$ & External port input \\
    $\thetavec$        & Trainable circuit parameters \\
    $\phi$             & Full parameter vector of circuit weights and read-out scale \\
    $\gamma$           & Per-node amplitude-damping rate (set by MINL ancilla angle $\theta_R$) \\
    $U_J(\thetavec)$   & Unitary quantum gate (IHM image of $\mathbf{J}$) \\
    $\mathcal{M}_R$    & MINL measurement map (IHM image of $\mathbf{R}$) \\
    $H_{\thetavec}$    & Quantum energy proxy $= \langle ZZ \rangle_{\thetavec}$ \\
    $\mathcal{E}(\rho)$ & CPTP map $= \sum_b K_b \rho K_b^\dagger$ \\
    $K_b$              & Kraus operator for measurement outcome $b$ \\
    $R_x, R_y, R_z$    & Single-qubit rotation gates \\
    $CZ$               & Controlled-$Z$ two-qubit gate \\
    $R_{ZZ}(\beta)$    & Parameterised $ZZ$ rotation gate \\
    $N_s$              & Number of training samples \\
    $N$                & Number of network nodes (QGNN) \\
    $L$                & Number of variational entanglement layers \\
    $K$                & Number of rollout integration steps \\
    $\Delta t$         & Integration time step \\
    $\varepsilon_E^{\mathrm{rel}}$ & Relative energy conservation error \\
    $f_{\mathrm{mono}}$ & Energy monotone fraction \\
    $E_{\mathrm{ps}}$  & Network phase-space energy $\tfrac12\sum_i(\langle X_i\rangle^2+\langle Y_i\rangle^2)$ \\
    \bottomrule
  \end{tabular}
  \caption{Notation used throughout the manuscript.}
  \label{tab:notation}
\end{table}

\end{document}